\documentclass{article}

\usepackage[nonatbib, final]{nips_2016}

\usepackage[utf8]{inputenc} % allow utf-8 input
\usepackage[T1]{fontenc}    % use 8-bit T1 fonts
\usepackage{hyperref}       % hyperlinks
\usepackage{url}            % simple URL typesetting
\usepackage{booktabs}       % professional-quality tables
\usepackage{amsfonts}       % blackboard math symbols
\usepackage{nicefrac}       % compact symbols for 1/2, etc.
\usepackage{microtype}      % microtypography

\usepackage{dsfont}
\usepackage{amsmath}
\usepackage{amsfonts}
\usepackage{amssymb}
\usepackage{amsthm}
\usepackage{subfigure}
\usepackage{epsfig}
\usepackage{color}
\usepackage{enumerate}
\usepackage{placeins}

\usepackage{times}
\usepackage{graphicx} % more modern
\usepackage{subfigure} 

\usepackage[numbers]{natbib}

\usepackage{algorithm}
\usepackage{algorithmic}

\usepackage{hyperref}

\newtheorem{lemma}{Lemma}
\newtheorem{theorem}{Theorem}
\newtheorem{corollary}{Corollary}
\newtheorem{definition}{Definition}
\newtheorem{assumption}{Assumption}
\newtheorem{proposition}{Proposition}
\newtheorem{remark}{Remark}

\theoremstyle{definition}

\newcommand{\X}{\mathcal{X}}

\newcommand{\Xspl}{X_{[n]}}

\newcommand{\real}{\mathbb{R}}

\newcommand{\norm}[1]{\left\|#1\right\|}

\newcommand{\braces}[1]{\left\{#1\right\}}
\DeclareMathOperator*{\argmax}{argmax}
\DeclareMathOperator*{\argmin}{argmin}

\title{Modal-set estimation with an application to clustering}

\author{
  Heinrich Jiang \thanks{Much of this work was done when this author was at Princeton University Mathematics Department.}\\
  \texttt{heinrich.jiang@gmail.com} \\
  \And
  Samory Kpotufe\\
  ORFE, Princeton University\\
  \texttt{samory@princeton.edu}\\
}

\begin{document}

\maketitle

\begin{abstract}
{\bf Abstract}. We present a first procedure that can estimate -- with statistical consistency guarantees -- any local-maxima of a density, under benign distributional conditions. The procedure estimates all such local maxima, or \emph{modal-sets}, of any bounded shape or dimension, including usual point-modes. In practice, modal-sets can arise as dense low-dimensional structures in noisy data, and more generally serve to better model the rich variety of locally-high-density structures in data. 

The procedure is then shown to be competitive on clustering applications, and moreover is quite stable to a wide range of settings of its tuning parameter. 
\end{abstract}

\section{Introduction}
Mode estimation is a basic problem in data analysis. 
 Modes, i.e. points of locally high density, serve as a measure of central tendency and are therefore important in unsupervised problems such as outlier detection, image or audio segmentation, and clustering in particular (as cluster cores). In the present work, we are interested in capturing a wider generality of \emph{modes}, i.e. general structures (other than single-points) of locally high density, that can arise in modern data. 

For example, application data in $\real^d$ (e.g. speech, vision) are often well modeled as arising from 
a lower-dimensional structure $M$ + noise. In other words, such data is densest on $M$, hence 
the ambient density $f$ is more closely modeled as locally maximal at (or near) $M$, a nontrivial subset of $\real^d$, 
rather than maximal only at single points in $\real^d$. Such a situation is illustrated in Figure \ref{fig:clustercores}. 

We therefore extend the notion of \emph{mode} to any connected 
subset of $\real^d$ where the unknown density $f$ is locally maximal; we refer to these as \emph{modal-sets} of $f$.  A modal-set can be of any bounded shape and dimension, from $0$-dimensional (point modes), to full dimensional surfaces, and aim to capture the possibly rich variety of dense structures in data. 

Our main contribution is a procedure, M(odal)-cores, that consistently estimates all such modal-sets from data, of general shape and dimension, with minimal assumption on the unknown $f$. The procedure builds on recent developments in topological data analysis 
\cite{SN10, CD10, KV11, RSNW12, balakrishnan2013cluster, CDKvL14, eldridge2015beyond}, and works by traversing certain $k$-NN graphs which encode level sets of a $k$-NN density estimate. We show that, if $f$ is continuous on compact support, the Hausdorff distance between any modal-set and its estimate vanishes as $n\to \infty$ (Theorem \ref{theo:main}); the estimation rate for point-modes matches (up to $\log n$) the known minimax rates. Furthermore, under mild additional smoothness condition on $f$ (H\"older continuity), \emph{false} structures (due to empirical variability) are correctly identified and pruned. We know of no such general statistical guarantees in mode estimation. 

While there is often a gap between theoretical procedures and practical ones, the present procedure is easy to implement and yields competitive scores on clustering applications; here, as in \emph{mode-based clustering}, clusters are simply defined as regions of high-density of the data, and the estimated modal-sets serve as the centers of these regions, i.e. as  \emph{cluster-cores}. A welcome aspect of the resulting clustering procedure is its stability to tuning 
settings of the parameter $k$ (from $k$-NN): it maintains high clustering scores (computed with knowledge of the ground-truth) over a wide range of settings of $k$, for various datasets. 
 Such stability to tuning is of practical importance, since typically the ground-truth is unknown, so clustering procedures  come with tuning parameters that are hard to set in practice. Practitioners therefore use various rule-of-thumbs and can thus benefit from procedures that are less-sensitive to their hyperparameters. 
 
 %Such stability is likely due to the fact that the procedure relies on mild distributional assumptions about $f$ and its local maxima, and is in fact consistent for a wide range of admissible values of $k$. 

In the next section we put our result in context with respect to previous work on mode estimation and density-based clustering in general.

\begin{figure} 
\centering 
\includegraphics[width=0.25\textwidth]{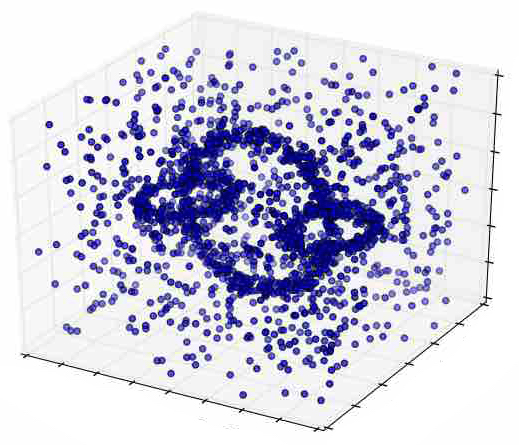}
\includegraphics[width=0.25\textwidth]{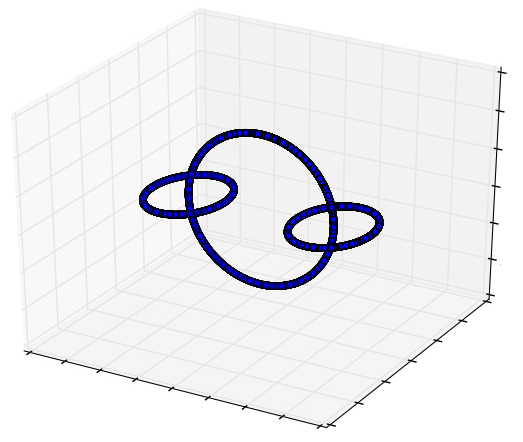}
\includegraphics[width=0.25\textwidth]{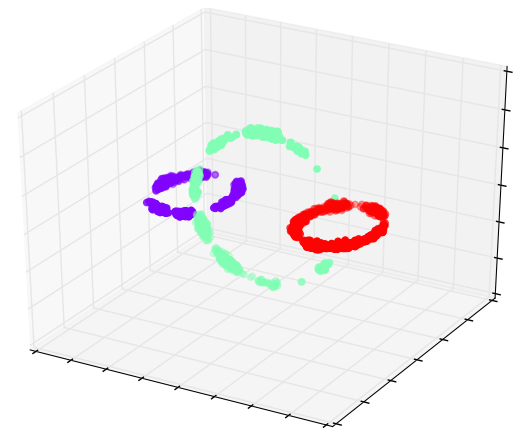}

\label{fig:clustercores}
\caption{Main phase of M-cores. (Left) Points-cloud generated as three 1-dimensional rings + noise. (Middle) The 3 rings, and (Right) their estimate (as modal-sets) by M-cores.}
\end{figure}

\subsection*{Related Work} 
$\bullet$ Much theoretical work on mode-estimation is concerned with understanding the statistical difficulty of the problem, and as such, often only considers the case of densities with single point-modes \cite{parzen1962estimation, chernoff1964estimation, eddy1980optimum, devroye1979recursive, tsybakov1990recursive, abraham2004asymptotic}. 
The more practical case of densities with multiple point-modes has received less attention in the theoretical literature. However there exist practical estimators, e.g., the popular \emph{Mean-Shift} procedure (which doubles as a clustering procedure), which are however harder to analyze. Recently, \cite{arias2013estimation} shows the consistency of a variant of Mean-Shift. Other recent work of \cite{genovese2013nonparametric} derives a method for pruning false-modes obtained by mode-seeking procedures. Also recent, the work of \cite{dasgupta2014optimal} shows that point-modes of a $k$-NN density estimate $f_k$ approximate the true modes of the unknown density $f$, assuming $f$ only has point-modes and bounded Hessian at the modes; their procedure, therefore operates on level-sets of $f_k$ (similar to ours), but fails in the presence of more general high-density structures such as modal-sets. To handle such general structures, we have to identify more appropriate level-sets to operate on, the main technical difficulty being that local-maxima of $f_k$ can be relatively far (in Hausdorff) from those of $f$, for instance single-point modes rather than more general modal-sets, due to data-variability. 
The present procedure handles general structures, and is consistent under the much weaker conditions of continuity (of $f$) on a compact domain. %The main algorithmic difficulty is that $f_k$ itself might only have point-modes (due to data-variability) even when $f$ has general modal-sets. 

A related line of work, which seeks more general structures than point-modes, is that of \emph{ridge} estimation (see e.g. \citep{ozertem2011locally, genovese2014nonparametric}). A ridge is typically defined as a lower-dimensional structure away from which the density curves (in some but not all directions), and can serve to capture various lower-dimensional patterns apparent in point clouds. In contrast, the modal-sets defined here can be full-dimensional and are always local maxima of the density. Also, unlike in ridge estimation, we do not require local differentiability of the unknown $f$, nor knowledge of the dimension of the structure, thus allowing a different but rich set of practical structures. 

$\bullet$ A main application of the present work, and of mode-estimation in general, is \emph{density-based clustering}.  Such clustering was formalized in early work of \cite{carmichael1968finding, hartigan1975clustering, H81}, and can take various forms, each with their advantage. 

In its hierarchical version, one is interested in estimating the connected components (CCs) of \emph{all} level sets 
$\braces{f \ge \lambda}_{\lambda>0}$ of the unknown density $f$. Many recent works analyze approaches that consistently estimate such a hierarchy under quite general conditions, e.g. \cite{SN10, CD10, KV11, RSNW12, balakrishnan2013cluster, CDKvL14, eldridge2015beyond}. 

In the \emph{flat} clustering version, one is interested in estimating the CCs of $\braces{f \ge \lambda}$ for a single $\lambda$, somehow appropriately chosen 
\citep{RV09, SSN09, MHL09, RW10, S11, sriperumbudur2012consistency}. The popular DBSCAN procedure  \citep{ester1996density} can be viewed as estimating such single level set. The main disadvantage here is in the ambiguity in the choice of $\lambda$, especially when the levels $\lambda$ of $f$ have different numbers of clusters (CCs). 

% but can also be viewed as clustering points around (estimated) modes of $f$. 

Another common flat clustering approach, most related to the present work, is \emph{mode-based} clustering. The approach clusters points to estimated modes of $f$, a fixed target, and therefore does away with the ambiguity in choosing an appropriate level $\lambda$ of $f$ \citep{fukunaga1975estimation, cheng1995mean, comaniciu2002mean, li2007nonparametric, chazal2013persistence}. As previously discussed, 
these approaches are however hard to analyze in that mode-estimation is itself not an easy problem. Popular examples are extensions of $k$-Means to categorical data \cite{chaturvedi2001k}, and the many variants of Mean-Shift which cluster points by gradient ascent to the closest mode. 
%We unfortunately can only highlight some of the recent results towards understanding such procedures. In particular, \cite{arias2013estimation} recently shows the consistency of a variant of Mean-Shift. 
Notably, the recent work \cite{wasserman2014feature} analyzes clustering error of Mean-Shift in a general high-dimensional setting with potentially irrelevant features. The main assumption is that $f$ only has point-modes.

\section{Overview of Results}
\label{sec:Overview}
%Overview 
\subsection{Basic Setup and Definitions}
We have samples $X_{[n]} = \{X_1,...,X_n\}$ drawn i.i.d. from 
a distribution $\mathcal{F}$ over $\mathbb{R}^d$ with density $f$. We let $\X$ denote the support of $f$. 
Our main aim is to estimate all local maxima of $f$, or \emph{modal-sets} of $f$, as we will soon define. 

We first require the following notions of distance between sets. 
\begin{definition}\label{ball} 
For $M\subset \X$, $x\in \X$, let $d(x, M) := \inf_{x'\in M} \norm{x - x'}$. 
The {\bf Hausdorff} distance between $A, B \subset \X$ is defined as 
$d(A,B) := \max \{ \sup_{x \in A} d(x, B), \sup_{y \in B} d(y, A) \}.$
\end{definition}

A modal set, defined below, extends the notion of a point-mode to general subsets of $\X$ 
where $f$ is locally maximal. These can arise for instance, as discussed earlier, in applications where high-dimensional 
data might be modeled as a (disconnected) manifold $\mathcal{M}$ + ambient noise, each connected component of which 
induces a modal set of $f$ in ambient space $\real^D$ (see e.g. Figure \ref{fig:clustercores}). 

\begin{definition}\label{modal-set} For any $M \subset \X$ and $r>0$, define the \emph{envelope} $B(M, r) := \{x : d(x, M) \leq r\}$. A connected set $M$ is a {\bf modal-set} of $f$ if  $\forall x \in M$, $f(x) = f_M$ for some fixed $f_M$, and there exist $r>0$ such that $f(x) < f_M$ for all $x \in B(M, r) \backslash M$. 
\end{definition}

\begin{remark} 
The above definition can be relaxed to \emph{$\epsilon_0$-modal sets}, i.e., to allow $f$ to vary by a small $\epsilon_0$ on $M$. Our results extend easily to this more relaxed definition, with minimal changes to some constants. This is because the procedure operates on $f_k$, and therefore already needs to account for variations in $f_k$  
on $M$. This is described in Appendix~\ref{appendix:eps-modal-set}. 
\end{remark}

\subsection{Estimating Modal-sets}
The algorithm relies on nearest-neighbor density estimate $f_k$, defined as follows. 

\begin{definition} \label{kNNdensity} Let $r_k(x) := \min \{ r : |B(x, r) \cap \Xspl | \ge k \}$. 
Define the {\bf $k$-NN density estimate} as
$$f_k(x) := \frac{k}{n\cdot v_d\cdot r_k(x)^d},
\text{where } v_d \text{ is the volume of a unit sphere in } \mathbb{R}^d.$$
\end{definition}

Furthermore, we need an estimate of the level-sets of $f$; various recent work on cluster-tree estimation (see e.g. \cite{CDKvL14}) have shown that such level sets are encoded by subgraphs of certain \emph{modified} $k$-NN graphs. 
Here however, we directly use $k$-NN graphs, simplifying implementation details, but requiring a bit of side analysis.

\begin{definition}
 Let $G(\lambda)$ denote the (mutual) {\bf $k$-NN graph} with vertices $\{ x \in X_{[n]} : f_k(x) \ge \lambda\}$ and an edge between $x$ and $x'$ iff $||x - x'|| \le \min \{r_k(x), r_k(x') \}$. 
 \end{definition}
 $G(\lambda)$ can be viewed as approximating the $\lambda$-level set of $f_k$, hence approximates the $\lambda$-level set of $f$ (implicit in the connectedness result in Appendix~\ref{appendix:integrality}).  
 
Algorithm \ref{alg:modalset} (M-cores) estimates the modal-sets of the unknown $f$.
It is based on various insights described below. 
%, while in a second phase, 
%Algorithm \ref{alg:modalset-clust} clusters the rest of points around estimated model-sets. 
%We next describe the main insights followed in Algorithm \ref{alg:modalset}, our main procedure.
A basic idea, used for instance in point-mode estimation \cite{dasgupta2014optimal}, is to proceed top-down on the level sets of $f_k$ (i.e. on $G(\lambda), \, \lambda \to 0$), and identify new modal-sets as they appear in separate CCs at a level $\lambda$. 

Here we have to however be careful: the CCs of $G(\lambda)$ (essentially modes of $f_k$) might be 
singleton points (since $f_k$ might take unique values over samples $x\in \Xspl$) while the modal-sets to be estimated might be of any dimension and shape. Fortunately, if a datapoint $x$, locally maximizes $f_k$, and belongs to some modal-set $M$ of $f$, then the rest of $M\cap \Xspl$ must be at a nearby level; Algorithm \ref{alg:modalset} therefore proceeds by checking a nearby level ($\lambda - 9\beta_k \lambda$) from which it picks a specific set of points as an estimate of $M$. The main parameter here is $\beta_k$ which is worked out explicitly in terms of $k$ and requires no a priori knowledge of distributional parameters. The confidence level $\delta$ can be viewed in practice as fixed (e.g. $\delta = 0.05$). The essential algorithmic parameter is therefore just $k$, which, as we will show, can be chosen over a wide range (w.r.t. $n$) while ensuring statistical consistency. 

\begin{definition}Let $0< \delta < 1$. Define $C_{\delta, n} := 16\log(2/\delta)\sqrt{d \log n}$, and define $\beta_k = 4\frac{C_{\delta, n}}{\sqrt{k}}$. \end{definition}

We note that the above definition of $\beta_k$ is somewhat conservative (needed towards theoretical guarantees), since the exact constants $C_{\delta, n}$ turn out to have little effect in implementation. 

A further algorithmic difficulty is that a level $G(\lambda)$ might have too many CCs w.r.t. the ground truth. For example, due to variability in the data, $f_k$ might have more modal-sets than $f$, inducing too many CCs at some level $G(\lambda)$. Fortunately, it can be shown that the nearby level $\lambda - 9\beta_k \lambda$ will likely have the right number of CCs. Such lookups down to lower-level act as a way of \emph{pruning false modal-sets}, and trace back to earlier work \citep{KV11} on pruning cluster-trees. Here, we need further care: 
we run the risk of over-estimating a given $M$ if we look too far down (aggressive pruning), since 
a CC at lower level might contain points \emph{far outside} of a modal-set $M$. 
Therefore, the main difficulty here is in figuring out how far down to look and yet not over-estimate \emph{any} $M$ (to ensure consistency). In particular our lookup \emph{distance} of $9\beta_k \lambda$ is adapted to the level $\lambda$ unlike in aggressive pruning. 

Finally, for clustering with M-cores, we can simply assign every data-point to the closest estimated modal-set (acting as cluster-cores). %This is done in Algorithm \ref{alg:modalset-clust}.   

\begin{algorithm}[tb]
   \caption{M-cores (estimating modal-sets).}
   \label{alg:modalset}
\begin{algorithmic}
   \STATE Initialize $\widehat{\mathcal{M}}:= \emptyset$. Define $\beta_k = 4\frac{C_{\delta, n}}{\sqrt{k}}$.   
   \STATE Sort the $X_i$'s in decreasing order of $f_k$ values (i.e. $f_k(X_i) \geq f_k(X_{i+1})$). 
   \FOR{$i=1$ {\bfseries to} $n$}
   \STATE Define $\lambda := f_k(X_i)$.
   \STATE Let $A$ be the CC of $G(\lambda - 9\beta_k \lambda)$ that contains $X_i$.  \hfill (\rm{i})
   \IF{$A$ is disjoint from all cluster-cores in $\widehat{\mathcal{M}}$}
   \STATE Add $\widehat{M} := \{ x \in A : f_k(x) > \lambda - \beta_k \lambda \}$ to
    $\widehat{\mathcal{M}}$. 
    \ENDIF
%   \FOR{$j=1$ {\bfseries to} $m$}
%        \STATE $\widehat{\mathcal{M}}:= \widehat{\mathcal{M}} \cup \{ \widehat{M}_j \}$
%   \ENDFOR
   \ENDFOR
   \STATE \textbf{return} $\widehat{\mathcal{M}}$. \textit{     // Each $\widehat M\in \widehat{\mathcal{M}}$ is a cluster-core estimating 
   a modal-set of the unknown $f$.}
%   \STATE \textit{Note that this is a set of sets of points. Each element of $\widehat{\mathcal{M}}$ will be a cluster-core. We take each cluster-core to be a predicted cluster. The remaining points (i.e. not in any cluster-core) can be discretionarily assigned to the clusters (i.e. closest cluster-core).}
\end{algorithmic}
\end{algorithm}

% \begin{algorithm}[tb]
%   \caption{Cluster the rest of $\Xspl$.}
%   \label{alg:modalset-clust}
%\begin{algorithmic}
%  \STATE Assign every $X_i\in \Xspl$ to $\argmin_{\widehat{M} \in \widehat{\mathcal{M}}} d(X_i, \widehat M)$. 
%\end{algorithmic}
%\end{algorithm}

\subsection{Consistency Results}
%We require some mild assumptions on the rate of decay of $f$ at the boundaries of modal-sets. 
%In particular the assumptions require modal-sets to be on the interior of $\X$. 
Our consistency results rely on the following mild assumptions. 
\begin{assumption} $f$ is continuous with compact support $\X$. Furthermore $f$ has a finite number of modal-sets 
all in the interior of its support $\X$. 
\label{assumption-main}
\end{assumption}
%When $M$ is clear from context, we drop the subscript on $u$ and $l$. 
%\end{assumption}

%As is common in work on mode or level-set estimation \cite{CD10, balakrishnan2013cluster, CDKvL14, dasgupta2014optimal}, our guarantees require that each modal-set be sufficiently \emph{salient} w.r.t. other modal-sets (i.e. there is a sufficiently wide \emph{valley} between them), and is also satisfied whenever $f$ is uniformly continuous \cite{CDKvL14}. 

We will express the convergence of the procedure explicitly in terms of quantities that characterize the behavior of $f$ at the boundary of every modal set. The first quantity has to do with how \emph{salient} a modal-set, i.e whether it is sufficiently \emph{separated} from other modal sets. We start with the following definition of \emph{separation}. 

\begin{definition}\label{rsalient} 
Two sets $A, A' \subset \mathcal{X}$ are {\bf $r$-separated}, if there exists a set $S$ such that every path from $A$
 to $A'$ crosses $S$ and $\sup_{x \in B(S, r)} f(x) < \inf_{x \in A\cup A'} f(x)$.
%A modal-set $M$ is {\bf $r$-salient} for $r > 0$ when $A_M$ is $r$-separated from $\mathcal{X}^{\lambda_M} \backslash A_M$ by $S_M$ for some set $S_M$.
\end{definition}
%Since we have assumed that $f$ has a finite-number of modal sets, we can simply consider a common saliency parameter $r_s$: 
%\begin{assumption} There exists $r_s>0$, such that \emph{all} modal-sets are ${r_s}$-{salient}. 
%\end{assumption} 
The next quantities characterize the \emph{change} in $f$ in a neighborhood of a modal set $M$. 
The existence of a proper such neighborhood $A_M$, and appropriate functions $u_M$ and $l_M$ capturing smoothness and curvature, follow from the above assumptions on $f$. This is captured in the proposition below. 

\begin{proposition}
\label{prop:main-assumptions}
Let $M$ be a modal-set of $f$. Then there exists a CC $A_M$ of some level-set $\mathcal{X}^{\lambda_M} := \{x : f(x) \ge \lambda_M\}$, containing $M$, such that the following holds. 
%there exists $\lambda_M, A_M, r_M, l_M, u_M, r_s, S_M$ such that the following holds. $A_M$ is a CC of  containing $M$ which satisifies the following
\begin{itemize} 
\item \emph{$A_M$ isolates $M$ by a valley}: $A_M$ does not intersect any other modal-set; and $A_M$ and $\mathcal{X}^{\lambda_M} \backslash A_M$ are $r_s$-separated (by some $S_M$) for some $r_s>0$ independent of $M$. 
\item \emph{$A_M$ is full-dimensional}: $A_M$ contains an envelope $B(M, r_M)$ of $M$, for some $r_M>0$.  
\item \emph{$f$ is both \emph{smooth} and has \emph{curvature} around $M$}: there exist functions $u_M$ and $l_M$, increasing and continuous on $[0, r_M]$, $u_M(0) = l_M(0) = 0$, such that $\forall x \in B(M, r_M)$, 
\begin{align*}
l_M(d(x, M)) \le f_M -  f(x) \le u_M(d(x, M)). 
\end{align*}

\end{itemize}
\end{proposition}

Finally, our consistency guarantees require the following admissibility condition on $k = k(n)$. 
This condition results, roughly, from needing the density estimate $f_k$ to properly approximate the behavior 
of $f$ in the neighborhood of a modal-set $M$. In particular, we intuitively need $f_k$ values to be smaller for 
points far from $M$ than for points close to $M$, and this should depend on the smoothness and curvature of $f$ 
around $M$ (as captured by $u_M$ and $l_M$). 

\begin{definition} $k$ is {\bf admissible} for a modal-set $M$ if (we let $u_M^{-1}$ denote the inverse of $u_M$):
$$\max \left\{ \left(\frac{24 \sup_{x \in \mathcal{X}} f(x)}{l_M(\min\{r_M, r_s\}/2)} \right)^2, 2^{7 + d}  \right\}\cdot C_{\delta, n}^2 \le k \le \frac{v_d \cdot f_M}{2^{2+2d}} \left(u_M^{-1} \left ( 
f_M\frac{C_{\delta, n}}{2\sqrt{k}}\right) \right)^d \cdot n.$$
\end{definition}

\begin{remark}\label{kadmissible} The admissibility condition on $k$, although seemingly opaque, allows for a wide range of settings of $k$. For example, suppose $u_M(t) = c t^{\alpha}$ for some $c, \alpha > 0$. These are polynomial tail conditions common in mode estimation, following e.g. from H\"older assumptions on $f$. 
Admissibility then (ignoring $\log (1/\delta)$), is immediately seen to correspond to the wide range
$$C_1\cdot \log n \leq k \leq C_2\cdot n^{2\alpha/(2\alpha + d)},$$ where $C_1, C_2$ are constants depending on $M$, but independent of $k$ and $n$. It's clear then that even the simple choice $k = \Theta(log^2 n)$ is always admissible \emph{for any} $M$ for $n$ sufficiently large.
\end{remark}

{\bf Main theorems.} We then have the following two main consistency results for Algorithm \ref{alg:modalset}. Theorem \ref{theo:main} states a rate (in terms of $l_M$ and $u_M$) at which 
any modal-set $M$ is approximated by some estimate in $\widehat{\mathcal{M}}$; Theorem \ref{pruning} establishes \emph{pruning} guarantees. 
\begin{theorem} 
\label{theo:main}
Let $0< \delta < 1$. The following holds with probability at least $1- 6\delta$, simultaneously for all modal-sets $M$ of $f$. Suppose $k$ is admissible for $M$. Then there exists $\widehat{M} \in \widehat{\mathcal{M}}$ such that the following holds. Let $l_M^{-1}$ denote the inverse of $l_M$. 
 \begin{align*}
 d(M, \widehat{M}) \le l_M^{-1}\left(\frac{8C_{\delta,n }}{\sqrt{k}}f_M\right), 
 \text{ which goes to } 0 \text{ as } C_{\delta, n}/\sqrt{k} \rightarrow 0.
 \end{align*}
\end{theorem}
If $k$ is admissible for all modal-sets $M$ of $f$, then $\widehat{\mathcal{M}}$ estimates all modal-sets of $f$ 
at the above rates. These rates can be instantiated under the settings in Remark~\ref{kadmissible}: 
suppose $l_M(t) = c_1 t^{\alpha_1}$, $u_M(t) = c t^\alpha$, $\beta_1 \geq \beta$; then 
the above bound becomes $d(M, \widehat{M}) \lesssim k^{-1/2{\alpha_1}}$ for admissible $k$. As in the remark, $k = \Theta (\log^2 n)$ is admissible, simultaneously for all $M$ (for $n$ sufficiently large), and therefore all modal-sets of $f$ are recovered at the above rate. 
In particular, taking large $k = O(n^{2\alpha/(2\alpha + d)})$ optimizes the rate to $O(n^{-\alpha/(2\alpha_1\alpha + \alpha_1 d)})$. Note that for $\alpha_1 = \alpha = 2$, the resulting rate ($n^{-1/(4+d)}$) is tight (see e.g. \cite{tsybakov1990recursive} for matching lower-bounds in the case of point-modes $M = \braces{x}$.).  

Finally, Theorem~\ref{pruning} (pruning guarantees) states that any estimated modal-set in $\widehat{\mathcal{M}}$, at a sufficiently high level (w.r.t. to $k$), corresponds to a \emph{true} modal-set of $f$ at a similar level. Its proof consists of showing that if two sets of points are wrongly disconnected at level $\lambda$, they remain connected at nearby level $\lambda - 9\beta_k \lambda$ (so are reconnected by the procedure). The main technicality is the dependence of the nearby level on the empirical $\lambda$; the proof is less involved and given in  Appendix~\ref{appendix:pruning}. 

\begin{theorem} \label{pruning}
Let $0< \delta < 1$. There exists $\lambda_0 = \lambda_0(n, k)$ such that the following holds with probability at least $1 - \delta$. All modal-set estimates in $\widehat{\mathcal{M}}$ chosen at level $\lambda \ge \lambda_0$ can be injectively mapped to modal-sets $\braces{M:  \lambda_M \geq \min_{\{x\in \mathcal{X}_{[n]} : f_k(x) \ge \lambda - \beta_k \lambda \}} f(x)}$, provided $k$ is admissible for all such $M$. 

In particular, if $f$ is H\"older-continuous, (i.e. $||f(x) - f(x')|| \le c||x - x'||^\alpha $ for some $0 < \alpha\le 1$, $c > 0$)
then $\lambda_0 \xrightarrow{n \to \infty} 0$, provided 
$C_1 \log n \le k \le C_2 n^{2\alpha / (2\alpha + d)}$, for some $C_1, C_2$ independent $n$. 

%(i.e. $||f(x) - f(x')|| \le c||x - x'||^\alpha $ for $0 < \beta\le 1$, $c > 0$), then there exists constants $C_1, C_2 > 0$ depending on $f$ such that when
%\begin{align*}
%
%\end{align*} 
%then $k$ is simultaneously admissible for all modal-sets and $\lambda_0 \rightarrow 0$ as $n, k\rightarrow \infty$. 

%{\color{red} In particular, if $f$ is H\"older-continuous (i.e. $||f(x) - f(x')|| \le c||x - x'||^\beta $ for $0 < \beta\le 1$, $c > 0$), then taking 
%$\tilde\epsilon = \Omega (\sqrt{\log n} \cdot (k/n)^{\beta /d})$ gives $\lambda_0  \rightarrow 0$ as $n, k \rightarrow \infty$} {\color{blue} If we set $\tilde \epsilon = \beta_k$ what does this condition become?, what is its source?} 
\end{theorem}
\begin{remark}
Thus with little additional smoothness ($\alpha \approx 0$) over uniform continuity of $f$, any estimate above level $\lambda_0 \to 0$ corresponds to a true modal-set of $f$. We note that these pruning guarantees can be strengthened as needed by implementing a more aggressive pruning: simply replace $G(\lambda - 9\beta_k \lambda)$ in the procedure (on line (\rm{i})) with $G(\lambda - 9\beta_k \lambda - \tilde\epsilon)$ using a \emph{pruning parameter} $\tilde \epsilon \ge 0$. This allows $\lambda_0 \rightarrow 0$ faster. However the rates of Theorem \ref{theo:main} (while maintained) then require a larger initial sample size $n$. 
This is discussed in Appendix~\ref{appendix:pruning}.  
\end{remark}

\section{Analysis Overview}
\label{sec:analysis}

The bulk of the analysis is in establishing Theorem~\ref{theo:main}. The key technicalities are in  
bounding distances from estimated cores to an unknown number of modal-sets of general shape, dimension and location.

The analysis considers each modal-set $M$ of $f$ separately, and only combines results  
in the end into the uniform consistency statement of Theorem \ref{theo:main}. 
The following notion of \emph{distance} from the sample $\Xspl$ to a modal-set $M$ will be crucial.
%Next, we define the distance from $M$ to its closest sample point. 
\begin{definition}
For any $x\in \X$, let $r_n(x) := d(\{x \}, \Xspl)$, and $r_n(M) := \sup_{x \in M} r_n(x)$. 
%\inf \{r' : B(x, r') \cap X_{[n]} \neq \emptyset \}$ 
\end{definition}

For each $M$, define $\hat{x}_M := \argmax_{x \in \mathcal{X}_M \cap X_{[n]}} f_k(x)$, a local maximizer of $f_k$ on the modal-set $M$. The analysis (concerning each $M$) proceeds in the following steps: 
\begin{itemize}
\item \emph{Isolation of $M$}: when processing $\hat{x}_M$, the procedure picks an estimate $\widehat{M}$ that contains no point from (or close to) modal-sets other than $M$.  
\item \emph{Integrality of $M$}: the estimate $\widehat{M}$ picks all of the envelope $B(M, r_n(M)) \cap X_{[n]}$. 
\item \emph{Consistency of $\widehat M$}: it can then be shown that $\widehat{M} \rightarrow M$ in \emph{Hausdorff} distance. 
This involves two directions: the first direction (that points of $M$ are close to $\widehat {M}$) follows from integrality; the second direction is to show that points in $\widehat {M}$ are close to $M$. 
\end{itemize}

The following gives an upper-bound on the distance from a modal-set to the closest sample point. It follows from Berstein-type VC concentration on masses of balls. The proof is given in  Appendix~\ref{supportinglemmas}. 

\begin{lemma}[Upper bound on $r_n$] \label{r_n_upper_bound}  Let $M$ be a modal-set with density $f_M$ and suppose that $k$ is admissible. With probability at least $1 - \delta$,
\begin{align*}
r_n(M) \le  \left(\frac{2C_{\delta, n} \sqrt{d \log n}}{n\cdot v_d\cdot f_M}\right)^{1/d}.
\end{align*}
\end{lemma}

We require a notion of a region $\X_M$ containing only points \emph{close} to $M$ but far from other modes. To this end, 
let $S_M$ denote the separating set from Definition~\ref{rsalient}. 

\begin{definition} 
 $\mathcal{X}_{M} := \{ x : \exists \text{ a path }\mathcal{P} \text{ from } x \text{ to } x'\in M \text{ such that } \mathcal{P} \cap S_M = \emptyset\}$. 
\end{definition}

\begin{lemma} [Isolation]\label{isolation}  Let $M$ be a modal-set and $k$ be admissible for $M$. 
Let $\hat{x}_M := \argmax_{x \in \mathcal{X}_M \cap X_{[n]}} f_k(x)$. Then the following holds
with probability at least $1-5\delta$. When processing sample point $\hat{x}_M$ in Algorithm~\ref{alg:modalset} we will add $\widehat{M}$ to $\widehat{\mathcal{M}}$ where 
%$\hat{x}_M \in \widehat{M}$ and 
$\widehat{M}$ does not contain points outside of $\mathcal{X}_M$.
\end{lemma}

\begin{lemma}[Integrality] \label{integrality} 
Let $M$ be a modal-set with density $f_M$, and suppose $k$ is admissible for $M$. Let $\hat{x}_M := \argmax_{x \in \mathcal{X}_M \cap X_{[n]}} f_k(x)$. Then the following holds with probability at least $1 - 3\delta$. When processing sample point $\hat{x}_M$ in Algorithm~\ref{alg:modalset}, if we add $\widehat{M}$ to $\widehat{\mathcal{M}}$, then $B(M, r_n(M)) \cap X_{[n]} \subseteq \widehat{M}$.
\end{lemma}
The proofs for Lemma~\ref{isolation} and~\ref{integrality} can be found Appendices~\ref{appendix:isolation} and~\ref{appendix:integrality}, respectively.

Combining isolation and integrality, we obtain:
\begin{corollary}[Identification] \label{identification}  Suppose we have the assumptions of Lemmas~\ref{isolation} and~\ref{integrality} for modal-set $M$. Define $\widehat{f}_M := \max_{x \in \mathcal{X}_M \cap X_{[n]}} f_k(x)$. With probability at least $1 - 5\delta$, there exists $\widehat{M} \in \widehat{\mathcal{M}}$ such that $B(M, r_n(M)) \cap X_{[n]} \subseteq  \widehat{M} \subseteq \{ x \in \mathcal{X}_M \cap X_{[n]} : f_k(x) \ge \widehat{f}_M - \beta_k \widehat{f}_M\}$.
\end{corollary}

Here, we give a sketch of the proof for Theorem 1 which can be found in Appendix~\ref{appendix:theomain}. 
\begin{proof}[Proof idea of Theorem 1] Define $\tilde{r} = l_M^{-1}\left(\frac{8C_{\delta,n }}{\sqrt{k}}f_M\right)$. There are two directions to show: $\max_{x \in \widehat{M}} d(x, M) \le \tilde{r}$ and $\sup_{x \in M} d(x, \widehat{M}) \le \tilde{r}$.

For the first direction, by Corollary~\ref{identification} we have $\widehat{M} \subseteq \{ x \in \mathcal{X}_M : f_k(x) \ge \widehat{f}_M - \beta_k \widehat{f}_M\}$ where $\widehat{f}_M := \max_{x \in \mathcal{X}_M \cap X_{[n]}} f_k(x)$. Thus, it suffices to show 
\begin{align}
\inf_{x \in B(M, r_n(M))} f_k(x) \ge \sup_{\mathcal{X}_M \backslash B(M, \tilde{r})} f_k(x) + \beta_k \widehat{f}_M.
\end{align}
Using known upper and lower bounds on $f_k$ in terms of $f$, we can lower bound the LHS by approximately $f_M - u_M(r)$ (for some $r < \tilde{r}$) and upper bound the first term on the RHS by approximately $f_M - l_M(\tilde{r})$. The remaining difficulty is carefully choosing an appropriate $r$.

For the other direction, by Corollary~\ref{identification}, $\widehat{M}$ contains all sample points in $B(M, r_n(M))$. Lemma~\ref{r_n_upper_bound} and the admissibility of $k$ implies that $r_n(x) \le \tilde{r}$ which easily gives us the result.
\end{proof}

\section{Experiments}
\label{sec:experiments}
%Experiments

\subsection{Practical Setup} 
%While in theory $\beta_k$ is set as in Algorithm \ref{alg:modalset}, these settings 
%are conservative so as to ensure consistency in general theoretical settings. 
%The \emph{pruning} parameter $\tilde \epsilon$, as shown earlier, just needs to 
%be sufficiently small as larger settings simply returns fewer cluster-cores.

The analysis prescribes a setting of $\beta_k = O(1/\sqrt{k})$. Throughout the experiments we simply fix $\beta_k = 2 / \sqrt{k}$, and let our choice of $k$ be the essential parameter. 
As we will see, M-cores yields competitive and stable performance for a wide-range of settings of $k$. The implementation can be done efficiently and is described in Appendix~\ref{appendix:implementation}.

%Given the $k$-NN density and neighbor sets, Algorithm~\ref{alg:modalset} can be implemented by considering the CCs of the mutual $k$-NN graph as a set of disjoint forests in a disjoint-set forest data structure {\color{red} this is not a common data-structure, \cite{}}. The computation complexity can be shown to be near-linear in $k$ and $n$ \cite{cormen}.

We will release an optimized Python/C++  version of the code at \cite{url}. 

%For the experiments of Section \ref{} we vary $k$ to get a sense of the stability of the procedure 

\subsection{Qualitative Experiments on General Structures} 
\begin{figure}[h]
\centering 
\includegraphics[width=0.155\textwidth]{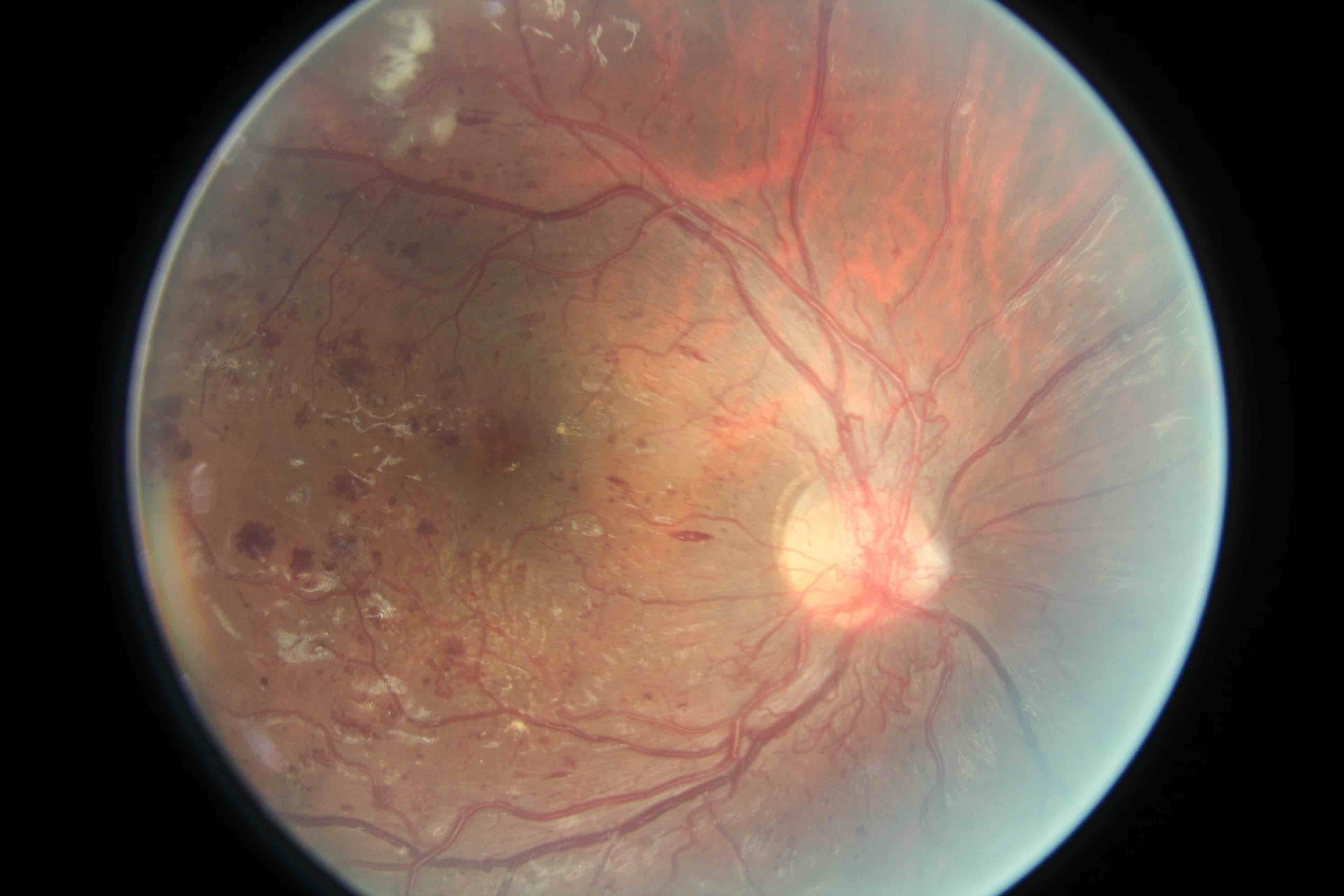}
\includegraphics[width=0.155\textwidth]{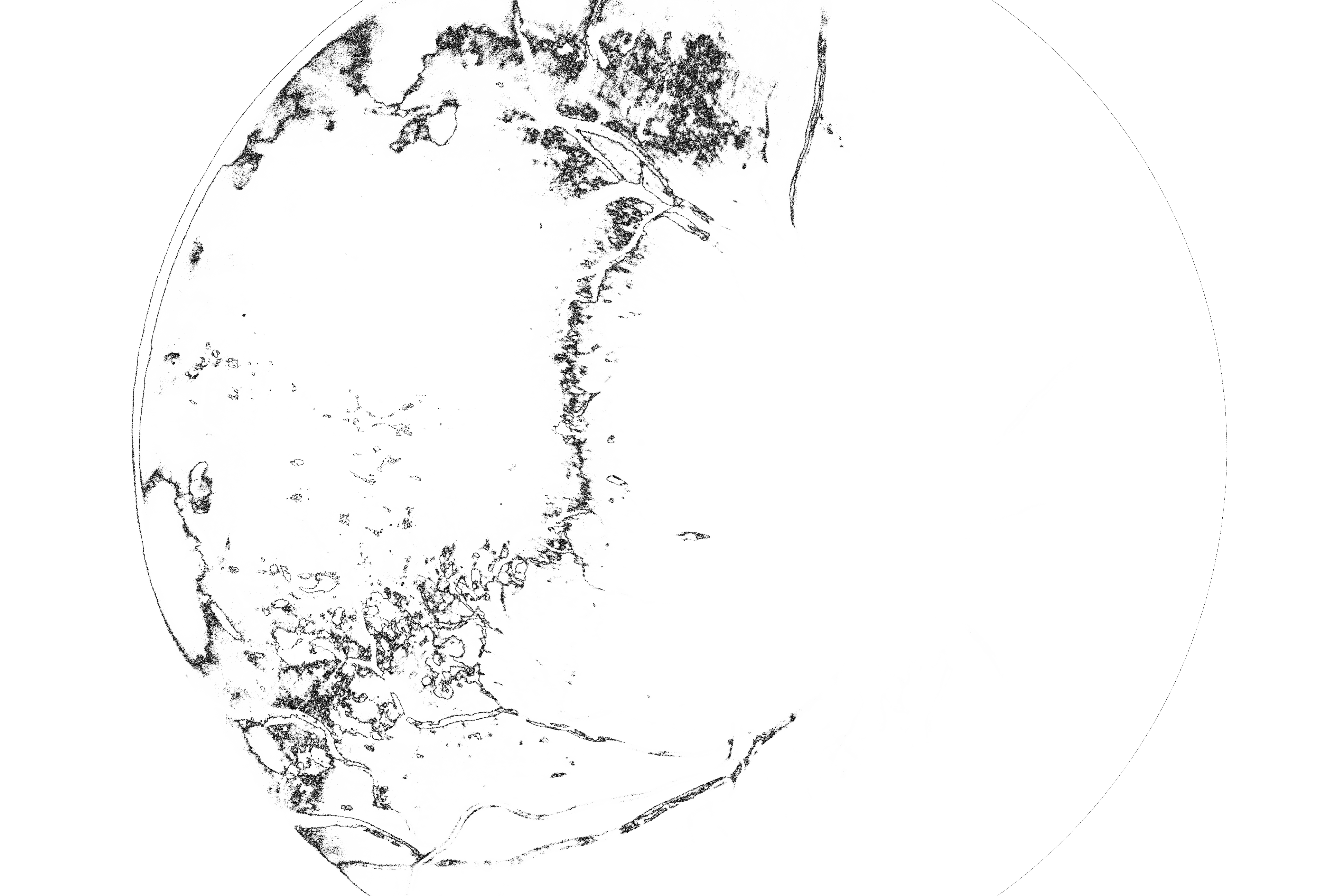}
\includegraphics[width=0.155\textwidth]{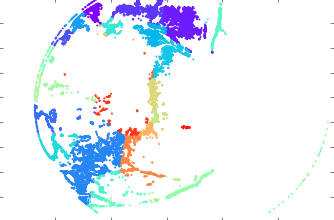}
\includegraphics[width=0.155\textwidth]{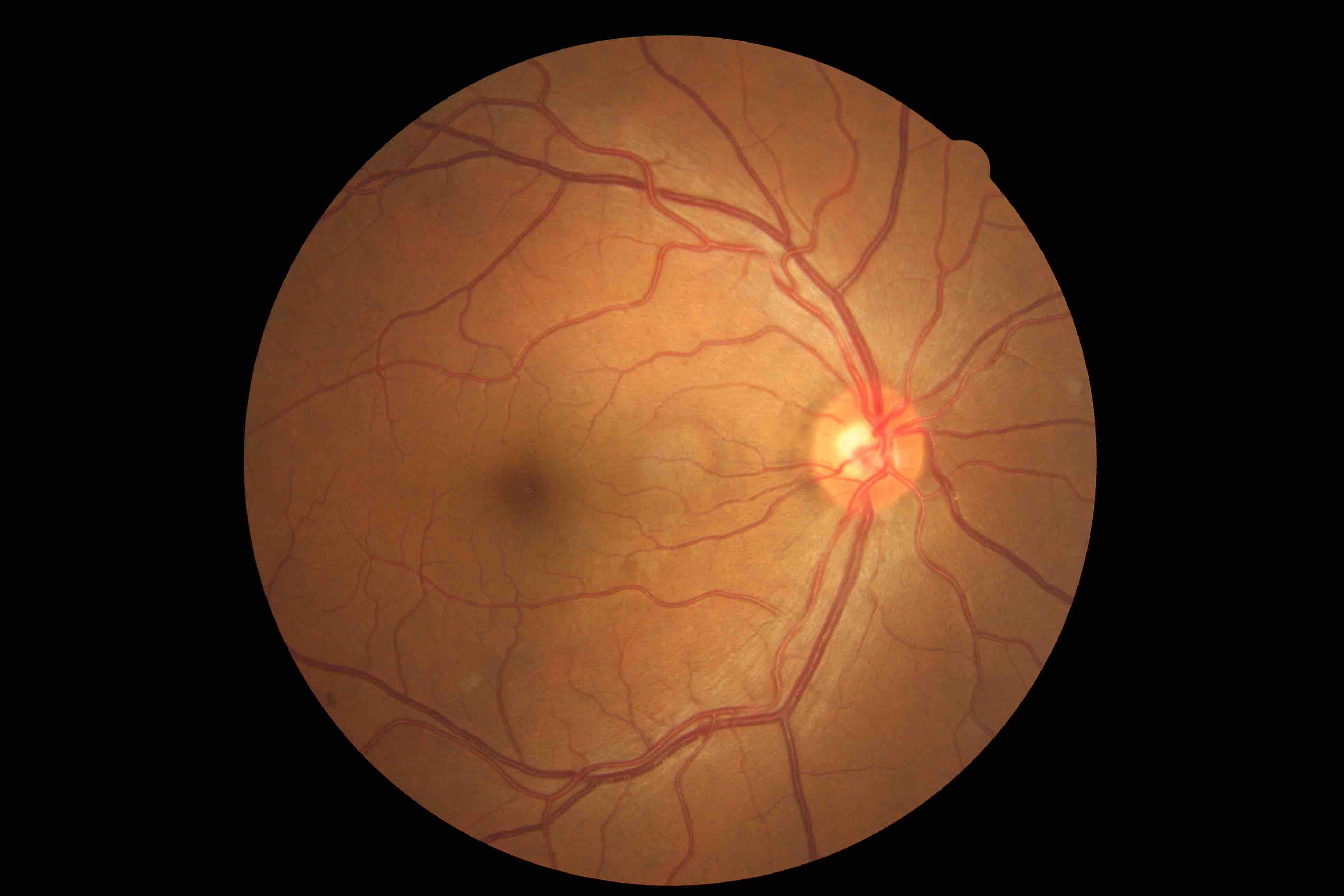}
\includegraphics[width=0.155\textwidth]{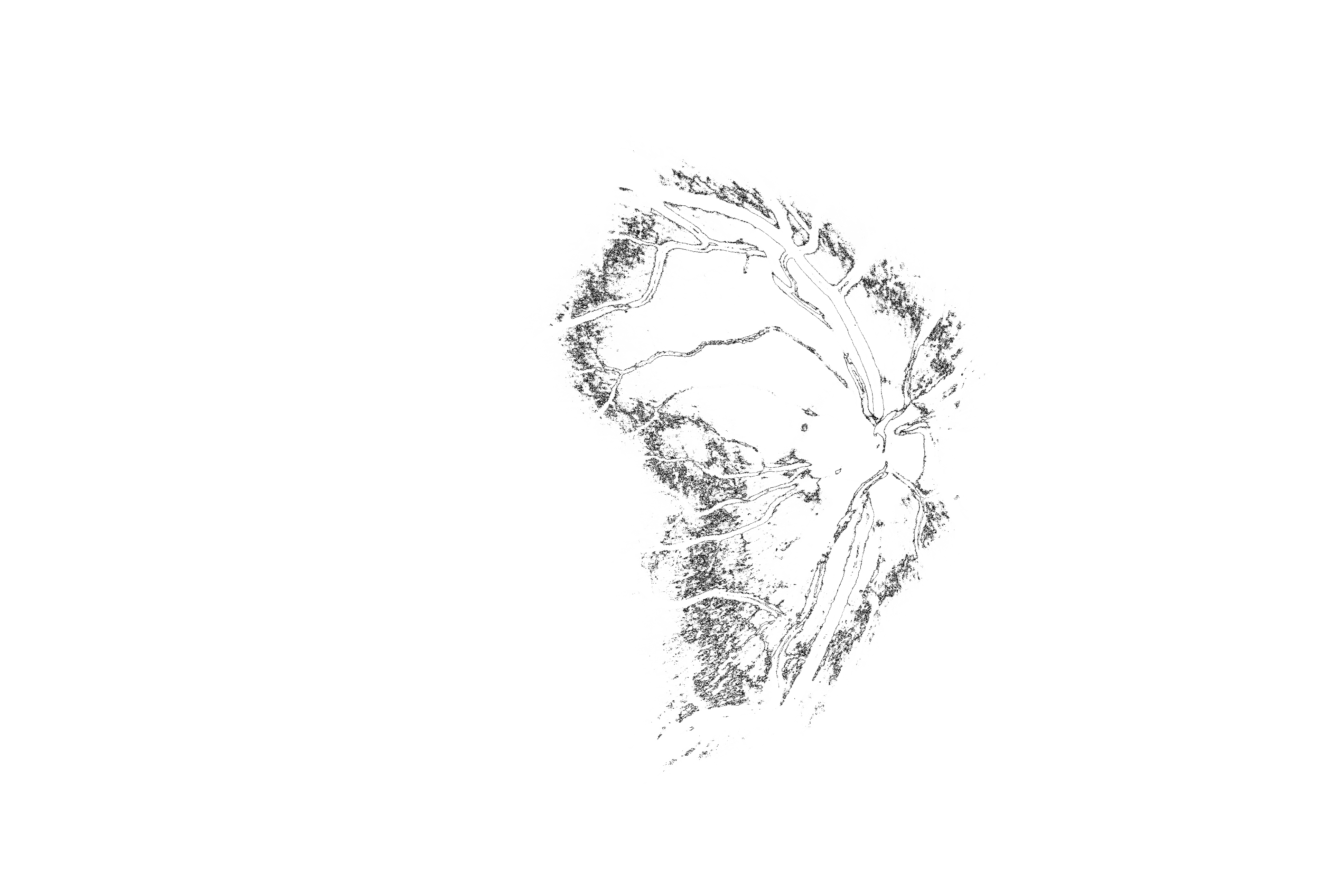}
\includegraphics[width=0.155\textwidth]{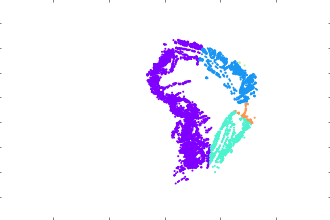}
\caption{Diabetic Retinopathy: (Left 3 figures) An unhealthy eye, (Right 3 figures) A healthy eye. 
In both cases, shown are (1) original image, (2) a filter applied to the image, (3) modal-sets (structures of capillaries)  estimated by M-cores on the corresponding filtered image. The unhealthy eye is characterized by a proliferation of 
damaged blood capillaries, while a healthy eye has visually fewer capillaries. The analysis task is to automatically discover the higher number of capillary-structures in the unhealthy eye. M-cores discovers $29$ structures for unhealthy eye vs $6$ for healthy eye. %This showcases the ability of the procedure to automatically estimate a reasonable number of \emph{clusters} of any shape. Here we use the default settings of M-cores described in the text.
}
\label{fig:eye}
\end{figure} 

We start with a qualitative experiment highlighting the flexibility of the procedure in fitting 
a large variety of high-density structures. For these experiments, we use $k = \frac{1}{2} \cdot \log^2 n$, which is within the
theoretical range for admissible values of $k$ (see Theorem \ref{theo:main} and Remark \ref{kadmissible}). 

We consider a medical imaging problem. Figure~\ref{fig:eye} displays the procedure applied to the Diabetic Retinopathy detection problem \cite{drd}. While this is by no means an end-to-end treatment of this detection problem, it gives a sense of M-cores' versatility in fitting 
real-world patterns. In particular, M-cores automatically estimates a reasonable number of clusters, independent 
of shape, while pruning away (most importantly in the case of the healthy eye) false clusters due to noisy data. As a result, it correctly picks up a much larger number of clusters in the case of the unhealthy eye.

\subsection{Clustering applications}

\begin{figure*}[ht]
\centering
\includegraphics[width=1\textwidth]{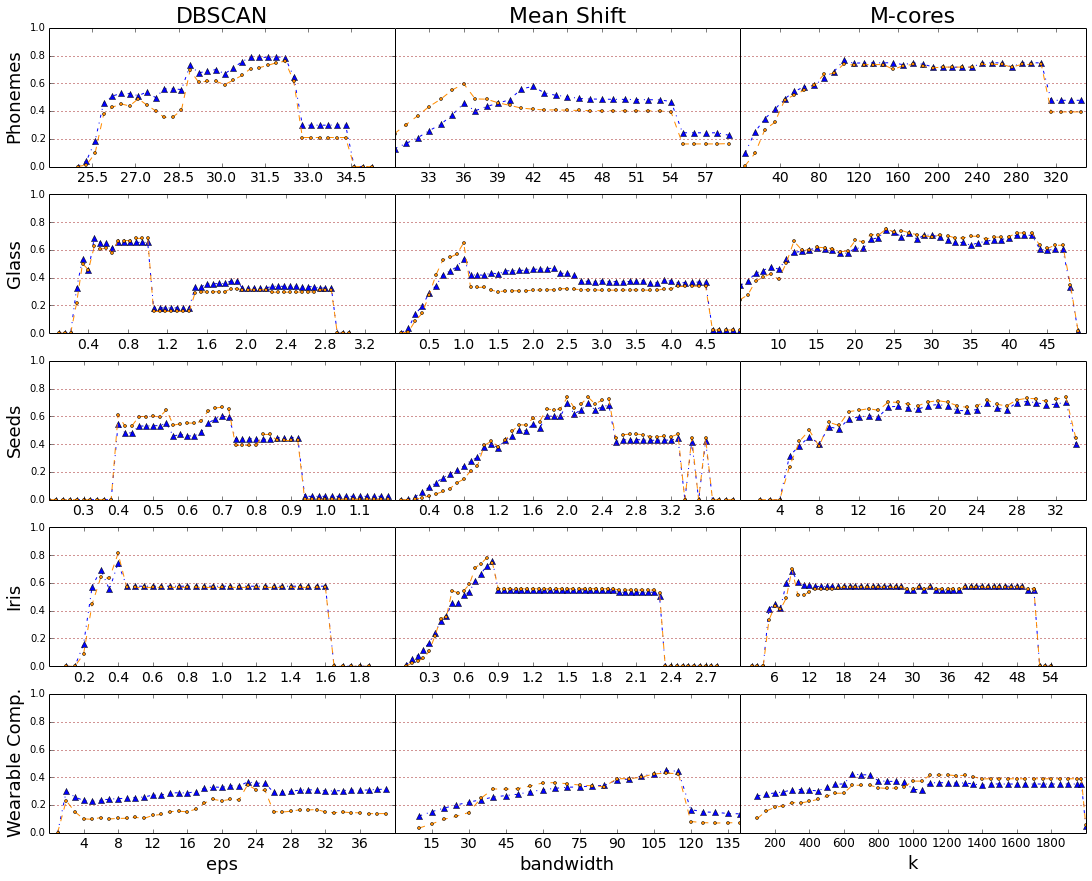}
\caption{Comparison on real datasets (along the rows) across different hyperparameter settings for each algorithm (along the columns). The hyperparameters being tuned are displayed at the bottom of the figure for each clustering algorithm. Scores: the blue line with triangular markers is Adjusted-Mutual-Information, and the dotted red line is Adjusted-Rand-Index. }
\label{fig:realworld}
\end{figure*} 
We now evaluate the performance of M-cores on clustering applications, where for
{\bf clustering:} we assign every point $x_i\in \Xspl$ to $\argmin_{\widehat{M} \in \widehat{\mathcal{M}}} d(x_i, \widehat M)$, i.e. to the closest estimated modal-set. 

We compare M-cores to two common density-based clustering procedures, DBSCAN and Mean-Shift, as implemented in the \textit{sci-kit-learn} package. Mean-Shift clusters data around point-modes, i.e. local-maxima of $f$, and is therefore most similar to M-cores in its objective. 

%DBSCAN clusters around level-sets with the added ambiguity of the choice of level-set (parameter \emph{eps}) potentially inducing a different number of clusters.

%Most importantly, we aim to understand whether the performance of the procedures is stable to varying the tuning parameters, since in practice the right settings are difficult to obtain without labels. To this end, we vary the 
%essential tuning parameters of each procedure and compare them under two established clustering scores 

{\bf Clustering scores.} We compute two established scores which evaluate a clustering against a labeled ground-truth. The \emph{rand-index}-score is the $0$-$1$ accuracy in grouping pairs of points, (see e.g. \cite{hubert}); the \emph{mutual information}-score is the (information theoretic) mutual-information between the distributions induced by the clustering and 
the ground-truth (each cluster is a mass-point of the distribution, see e.g. \cite{vinh2010information}). For both scores we report the \emph{adjusted} version, which adjusts the score so that a random clustering (with the same number of clusters as the ground-truth) scores near $0$ (see e.g. \cite{hubert}, \cite{vinh2010information}). 

{\bf Datasets.} Phonemes \cite{hastie2005elements}, and UCI datasets: Glass, Seeds, Iris, and Wearable Computing. They are described in the table below.% \cite{Lichman:2013}. 
\begin{center}
    \begin{tabular}{ | p{2cm}| p{0.7cm} | p{0.35cm} | p{0.7cm} | p{8cm} |}
    \hline
    {\small \emph{Dataset}} & $n$ & $d$ & {\small \emph{Labels}} & {\small \emph{Description}}  \\ \hline
    {\small Phonemes} & {\small 4509} & {\small 256} & {\small 5} & {\small Log-periodograms of spoken phonemes} \\ \hline
    {\small Glass} & {\small 214} & {\small 7} & {\small 6} & \small{Properties of different types of glass} \\ \hline
    {\small Seeds} & {\small 210} & {\small 7} & {\small 3} & \small{Geometric measurements of wheat kernels} \\ \hline
    {\small Iris} & {\small 150} & {\small 4} & {\small 3}  & \small{Various measurements over species of flowers} \\ \hline
    {\small Wearable} & {\small 10000} & {\small 12} & {\small 5}  & 
    \small{4 sensors on a human body, recording body posture and activity} \\ \hline
    \end{tabular}
\end{center}

{\bf Results.} Figure~\ref{fig:realworld} reports the performance of the procedures for each dataset. Rather than reporting the performance of the 
procedures under \emph{optimal-tuning}, we report their performance \emph{over a range} of hyperparameter settings, 
mindful of the fact that optimal-tuning is hardly found in practice (this is a general problem in clustering given the lack of ground-truth to guide tuning).

For M-cores we vary the parameter $k$. For DBSCAN and Mean-Shift, we vary the main parameters, respectively \emph{eps} (choice of level-set), and \emph{bandwidth} (used in density estimation). 
M-cores yields competitive performance across the board, with stable scores over a large range of values of $k$ (relative to sample size). Such stable performance to large changes in $k$ is quite desirable, considering that proper tuning of hyperparameters remains a largely open problem in clustering. 
\vspace{0.3cm}

{\bf Conclusion} \\

\vspace{-0.2cm} 

We presented a theoretically-motivated procedure which can consistently estimate modal-sets, i.e. nontrivial high-density structures in data, under benign distributional conditions. This procedure is easily implemented and yields competitive and stable scores in clustering applications.

{
%\subsubsection*{References}
%{
\bibliography{refs}

\begin{thebibliography}{10}

\bibitem{SN10}
W.~Stuetzle and R.~Nugent.
\newblock A generalized single linkage method for estimating the cluster tree
  of a density.
\newblock {\em Journal of Computational and Graphical Statistics},
  19(2):397--418, 2010.

\bibitem{CD10}
K.~Chaudhuri and S.~Dasgupta.
\newblock Rates for convergence for the cluster tree.
\newblock In {\em Advances in Neural Information Processing Systems}, 2010.

\bibitem{KV11}
S.~Kpotufe and U.~von Luxburg.
\newblock Pruning nearest neighbor cluster trees.
\newblock In {\em International Conference on Machine Learning}, 2011.

\bibitem{RSNW12}
A.~Rinaldo, A.~Singh, R.~Nugent, and L.~Wasserman.
\newblock Stability of density-based clustering.
\newblock {\em Journal of Machine Learning Research}, 13:905--948, 2012.

\bibitem{balakrishnan2013cluster}
S.~Balakrishnan, S.~Narayanan, A.~Rinaldo, A.~Singh, and L.~Wasserman.
\newblock Cluster trees on manifolds.
\newblock In {\em Advances in Neural Information Processing Systems}, pages
  2679--2687, 2013.

\bibitem{CDKvL14}
K.~Chaudhuri, S.~Dasgupta, S.~Kpotufe, and U.~von Luxburg.
\newblock Consistent procedures for cluster tree estimation and pruning.
\newblock {\em Arxiv}, 2014.

\bibitem{eldridge2015beyond}
Justin Eldridge, Yusu Wang, and Mikhail Belkin.
\newblock Beyond hartigan consistency: Merge distortion metric for hierarchical
  clustering.
\newblock {\em Conference on Learning Theory}, 2015.

\bibitem{parzen1962estimation}
Emanuel Parzen et~al.
\newblock On estimation of a probability density function and mode.
\newblock {\em Annals of mathematical statistics}, 33(3):1065--1076, 1962.

\bibitem{chernoff1964estimation}
Herman Chernoff.
\newblock Estimation of the mode.
\newblock {\em Annals of the Institute of Statistical Mathematics},
  16(1):31--41, 1964.

\bibitem{eddy1980optimum}
William~F Eddy et~al.
\newblock Optimum kernel estimators of the mode.
\newblock {\em The Annals of Statistics}, 8(4):870--882, 1980.

\bibitem{devroye1979recursive}
Luc Devroye.
\newblock Recursive estimation of the mode of a multivariate density.
\newblock {\em Canadian Journal of Statistics}, 7(2):159--167, 1979.

\bibitem{tsybakov1990recursive}
Aleksandr~Borisovich Tsybakov.
\newblock Recursive estimation of the mode of a multivariate distribution.
\newblock {\em Problemy Peredachi Informatsii}, 26(1):38--45, 1990.

\bibitem{abraham2004asymptotic}
Christophe Abraham, G{\'e}rard Biau, and Beno{\^\i}t Cadre.
\newblock On the asymptotic properties of a simple estimate of the mode.
\newblock {\em ESAIM: Probability and Statistics}, 8:1--11, 2004.

\bibitem{arias2013estimation}
Ery Arias-Castro, David Mason, and Bruno Pelletier.
\newblock On the estimation of the gradient lines of a density and the
  consistency of the mean-shift algorithm.
\newblock {\em Unpublished Manuscript}, 2013.

\bibitem{genovese2013nonparametric}
Christopher Genovese, Marco Perone-Pacifico, Isabella Verdinelli, and Larry
  Wasserman.
\newblock Nonparametric inference for density modes.
\newblock {\em arXiv preprint arXiv:1312.7567}, 2013.

\bibitem{dasgupta2014optimal}
Sanjoy Dasgupta and Samory Kpotufe.
\newblock Optimal rates for k-nn density and mode estimation.
\newblock In {\em Advances in Neural Information Processing Systems}, pages
  2555--2563, 2014.

\bibitem{ozertem2011locally}
Umut Ozertem and Deniz Erdogmus.
\newblock Locally defined principal curves and surfaces.
\newblock {\em The Journal of Machine Learning Research}, 12:1249--1286, 2011.

\bibitem{genovese2014nonparametric}
Christopher~R Genovese, Marco Perone-Pacifico, Isabella Verdinelli, Larry
  Wasserman, et~al.
\newblock Nonparametric ridge estimation.
\newblock {\em The Annals of Statistics}, 42(4):1511--1545, 2014.

\bibitem{carmichael1968finding}
JW~Carmichael, J~Alan George, and RS~Julius.
\newblock Finding natural clusters.
\newblock {\em Systematic Zoology}, pages 144--150, 1968.

\bibitem{hartigan1975clustering}
John~A Hartigan.
\newblock {\em Clustering algorithms}.
\newblock John Wiley \& Sons, Inc., 1975.

\bibitem{H81}
J.A. Hartigan.
\newblock Consistency of single linkage for high-density clusters.
\newblock {\em Journal of the American Statistical Association},
  76(374):388--394, 1981.

\bibitem{RV09}
P.~Rigollet and R.~Vert.
\newblock Fast rates for plug-in estimators of density level sets.
\newblock {\em Bernoulli}, 15(4):1154--1178, 2009.

\bibitem{SSN09}
A.~Singh, C.~Scott, and R.~Nowak.
\newblock Adaptive hausdorff estimation of density level sets.
\newblock {\em Annals of Statistics}, 37(5B):2760--2782, 2009.

\bibitem{MHL09}
M.~Maier, M.~Hein, and U.~von Luxburg.
\newblock Optimal construction of k-nearest neighbor graphs for identifying
  noisy clusters.
\newblock {\em Theoretical Computer Science}, 410:1749--1764, 2009.

\bibitem{RW10}
A.~Rinaldo and L.~Wasserman.
\newblock Generalized density clustering.
\newblock {\em Annals of Statistics}, 38(5):2678--2722, 2010.

\bibitem{S11}
I.~Steinwart.
\newblock Adaptive density level set clustering.
\newblock In {\em 24th Annual Conference on Learning Theory}, 2011.

\bibitem{sriperumbudur2012consistency}
Bharath~K Sriperumbudur and Ingo Steinwart.
\newblock Consistency and rates for clustering with dbscan.
\newblock In {\em International Conference on Artificial Intelligence and
  Statistics}, pages 1090--1098, 2012.

\bibitem{ester1996density}
Martin Ester, Hans-Peter Kriegel, J{\"o}rg Sander, and Xiaowei Xu.
\newblock A density-based algorithm for discovering clusters in large spatial
  databases with noise.
\newblock In {\em Kdd}, volume~96, pages 226--231, 1996.

\bibitem{fukunaga1975estimation}
Keinosuke Fukunaga and Larry Hostetler.
\newblock The estimation of the gradient of a density function, with
  applications in pattern recognition.
\newblock {\em Information Theory, IEEE Transactions on}, 21(1):32--40, 1975.

\bibitem{cheng1995mean}
Yizong Cheng.
\newblock Mean shift, mode seeking, and clustering.
\newblock {\em Pattern Analysis and Machine Intelligence, IEEE Transactions
  on}, 17(8):790--799, 1995.

\bibitem{comaniciu2002mean}
Dorin Comaniciu and Peter Meer.
\newblock Mean shift: A robust approach toward feature space analysis.
\newblock {\em Pattern Analysis and Machine Intelligence, IEEE Transactions
  on}, 24(5):603--619, 2002.

\bibitem{li2007nonparametric}
Jia Li, Surajit Ray, and Bruce~G Lindsay.
\newblock A nonparametric statistical approach to clustering via mode
  identification.
\newblock {\em Journal of Machine Learning Research}, 8(8), 2007.

\bibitem{chazal2013persistence}
Fr{\'e}d{\'e}ric Chazal, Leonidas~J Guibas, Steve~Y Oudot, and Primoz Skraba.
\newblock Persistence-based clustering in riemannian manifolds.
\newblock {\em Journal of the ACM (JACM)}, 60(6):41, 2013.

\bibitem{chaturvedi2001k}
Anil Chaturvedi, Paul~E Green, and J~Douglas Caroll.
\newblock K-modes clustering.
\newblock {\em Journal of Classification}, 18(1):35--55, 2001.

\bibitem{wasserman2014feature}
Larry Wasserman, Martin Azizyan, and Aarti Singh.
\newblock Feature selection for high-dimensional clustering.
\newblock {\em arXiv preprint arXiv:1406.2240}, 2014.

\bibitem{url}
M-cores code release.
\newblock \url{https://github.com/hhjiang/mcores}.

\bibitem{drd}
Diabetic retinopathy detection.
\newblock \url{https://www.kaggle.com/c/diabetic-retinopathy-detection}.

\bibitem{hubert}
Lawrence Hubert and Phipps Arabie.
\newblock Comparing partitions.
\newblock {\em Journal of classification}, 2(1):193--218, 1985.

\bibitem{vinh2010information}
Nguyen~Xuan Vinh, Julien Epps, and James Bailey.
\newblock Information theoretic measures for clusterings comparison: Variants,
  properties, normalization and correction for chance.
\newblock {\em The Journal of Machine Learning Research}, 11:2837--2854, 2010.

\bibitem{hastie2005elements}
Trevor Hastie, Robert Tibshirani, Jerome Friedman, and James Franklin.
\newblock The elements of statistical learning: data mining, inference and
  prediction.
\newblock {\em The Mathematical Intelligencer}, 27(2):83--85, 2005.

\end{thebibliography}
\bibliographystyle{unsrt}

%}
}

\newpage
\appendix

As discussed in the main text, the results are easily extended to handle more general modal-sets where the density can vary by $\epsilon_0$. We therefore will be showing such more general results which directly imply the results in the main text. 

We give a generalization of modal-sets where the density is allowed to vary by $\epsilon_0 \ge 0$, called $\epsilon_0$-modal sets, which will be defined shortly. In order to estimate the $\epsilon_0$-modal sets, we derive Algorithm~\ref{alg:epsilonmodalset}, which is a simple generalization of Algorithm~\ref{alg:modalset}. Algorithm~\ref{alg:modalset} is Algorithm~\ref{alg:epsilonmodalset} with the setting $\epsilon_0 = 0$ and $\tilde{\epsilon} = 0$. Changing $\tilde{\epsilon}$ to larger values will allow us prune false modal-sets away more aggressively, which will be discussed in Appendix~\ref{appendix:pruning}. 

Throughout the Appendix, we restate analogues of the results in the main text for the more general $\epsilon_0$-modal sets and Algorithm~\ref{alg:epsilonmodalset}. It will be understood that these results will imply the results in the main text with the setting $\epsilon_0 = 0$ and $\tilde{\epsilon} = 0$. 

In Appendix~\ref{appendix:pointcloud} we formalize common situations well-modeled by modal-sets.
In Appendix~\ref{appendix:implementation} we give implementation details. 

\section{$\epsilon_0$-modal sets}\label{appendix:eps-modal-set}

\begin{definition}\label{eps-modal-set} For $\epsilon_0 \ge 0$, connected set $M$ is an {\bf $\epsilon_0$-modal set} of $f$ if there exists $f_M > \epsilon_0$ such that $\sup_{x\in M} f(x) = f_M$ and $M$ is a CC of the level set ${\mathcal{X}^{f_M - \epsilon_0}} := \{ x : f(x) \ge f_M - \epsilon_0 \}$.
\end{definition}

We require the following Assumption~\ref{assumption2} on $\epsilon_0$-modal sets. Note that under Assumption~\ref{assumption-main} on modal-sets, Assumption~\ref{assumption2} on $\epsilon_0$-modal sets will hold for $\epsilon_0$ sufficiently small.

\begin{assumption}\label{assumption2}
The $\epsilon_0$-modal sets are on the interior of $\mathcal{X}$ and $f_M \ge 2\epsilon_0$ for all $\epsilon_0$-modal sets $M$. 
\end{assumption}

\begin{remark}
Since each $\epsilon_0$-modal set contains a modal-set, it follows that the number of $\epsilon_0$-modal sets is finite.
\end{remark}

The following extends Proposition~\ref{prop:main-assumptions} to show the additional properties of the regions around the $\epsilon_0$-modal sets necessary in our analysis. The proof is in Appendix~\ref{supportinglemmas}.

%\textbf{.}
\begin{proposition}[Extends Proposition~\ref{prop:main-assumptions}] \label{prop:main-assumptions-general}
For any $\epsilon_0$-modal set $M$, there exists $\lambda_M, A_M, r_M, l_M, u_M, r_s, S_M$ such that the following holds. $A_M$ is a CC of $\mathcal{X}^{\lambda_M} := \{x : f(x) \ge \lambda_M\}$ containing $M$ which satisifies the following.
\begin{itemize} 
\item \emph{$A_M$ isolates $M$ by a valley}: $A_M$ does not intersect any other $\epsilon_0$-modal sets and $A_M$ and $\mathcal{X}^{\lambda_M} \backslash A_M$ are $r_s$-separated by $S_M$ with $r_s>0$ where $r_s$ does not depend on $M$. 

\item \emph{$A_M$ is full-dimensional}: $A_M$ contains an envelope $B(M, r_M)$ of $M$, with $r_M>0$.  
\item \emph{$f$ is smooth around some maximum modal-set in $M$}: There exists modal-set $M_0 \subseteq M$ such that $f$ has density $f_M$ on $M_0$ and $f_M - f(x) \le u_M(d(x, M_0))$ for $x \in B(M_0, r_M)$
\item \emph{$f$ is both \emph{smooth} and has \emph{curvature} around $M$}: $u_M$ and $l_M$ are increasing continuous functions on $[0, r_M]$, $u_M(0) = l_M(0) = 0$ and $u_M(r), l_M(r) > 0$ for $r > 0$, and
\begin{align*}
l_M(d(x, M)) \le f_M -  \epsilon_0 - f(x) \le u_M(d(x, M))  \forall x \in B(M, r_M).
\end{align*}

\end{itemize}
\end{proposition}

\begin{algorithm}[tb]
   \caption{M-cores (estimating $\epsilon_0$-modal-sets)}
   \label{alg:epsilonmodalset}
\begin{algorithmic}
   \STATE Initialize $\widehat{\mathcal{M}}:= \emptyset$. Define $\beta_k = 4\frac{C_{\delta, n}}{\sqrt{k}}$. 
   \STATE Sort the $X_i$'s in descending order of $f_k$ values. 
   \FOR{$i=1$ {\bfseries to} $n$}
   \STATE Define $\lambda := f_k(X_i)$.
   \STATE Let $A$ be the CC of $G(\lambda - 9\beta_k \lambda - \epsilon_0 - \tilde{\epsilon})$ that contains $X_i$.
   \IF{$A$ is disjoint from all cluster-cores in $\widehat{\mathcal{M}}$}
   \STATE Add $\widehat{M} := \{ x \in A : f_k(x) > \lambda - \beta_k \lambda  - \epsilon_0 \}$ to
    $\widehat{\mathcal{M}}$. 
    \ENDIF
   \ENDFOR
   \STATE \textbf{return} $\widehat{\mathcal{M}}$. 

\end{algorithmic}
\end{algorithm}

Next we give admissibility conditions for $\epsilon_0$-modal sets. The only changes (compared to admissibility conditions for modal-sets) are the constant factors. In particular, when $\epsilon_0=0$ and $\tilde\epsilon = 0$ it is the admissibility conditions for modal-sets. As discussed in the main text, a larger $\tilde\epsilon$ value will prune more aggressively at the cost of requiring a larger number of samples. Furthermore, it is implicit below that $\tilde\epsilon < l_M(\min\{r_M, r_s\}/2)$. This ensures that we don't prune too aggressively that the estimated $\epsilon_0$-modal sets merge together.

\begin{definition} $k$ is {\bf admissible} for an $\epsilon_0$-modal set $M$ if (letting $u_M^{-1}, l_M^{-1}$ be the inverses of $u_M, l_M$)
\begin{align*}
\max \left\{ \left(\frac{24C_{\delta, n} (\sup_{x \in \mathcal{X}} f(x) + \epsilon_0)}{l_M(\min\{r_M, r_s\}/2) - \tilde\epsilon}  \right)^2, 2^{7 + d} C_{\delta, n}^2 \right\} \\
\le k \le \frac{v_d \cdot (f_M - \epsilon_0)}{2^{2+2d}} \left(u_M^{-1} \left ( \frac{C_{\delta, n} (f_M - \epsilon_0)}{2\sqrt{k}}\right) \right)^d \cdot n.
\end{align*}
\end{definition}

\section{Supporting lemmas and propositions}\label{supportinglemmas}

\begin{proof}[Proof of Proposition \ref{prop:main-assumptions-general}]
Let $M$ be an $\epsilon_0$-modal set with maximum density $f_M$ and minimum density $f_M - \epsilon_0$ (i.e. $f_M - \epsilon_0 \le f(x) \le f_M$ for $x\in M$). 
Define ${\mathcal{X}^{\lambda}} := \{ x : f(x) \ge \lambda\}$.
Let $A_1,...,A_m$ be the CCs of ${\mathcal{X}^{f_M - \epsilon_0}}$ (there are a finite number of CCs since each CC contains at least one modal-set and the number of modal-sets is finite). 
Define $r_{\text{min}} :=  \min_{A_i \neq A_j} \inf_{x \in A_i, x' \in A_j} |x - x'|$, which is the minimum distance between pairs of points in different CCs.
Next, define the one-sided Hausdorff distance for closed sets $A, B$: $d_{H'}(A, B) := \max_{x \in A} \min_{x \in B} |x - y|$. Then consider
$g(t) := d_{H'}({\mathcal{X}^{f_M - \epsilon_0 - t}}, {\mathcal{X}^{f_M - \epsilon_0}})$.

Since $f$ is continuous and has a finite number of modal-sets, $g$ has a finite number of points of discontinuity (i.e. when $f_M - \epsilon_0 - t$ is the density of some modal-set) and we have $g(t) \rightarrow 0$ as $t \rightarrow 0$. 
Thus, there exists $0 < \lambda_M < f_M - \epsilon_0$ such that $g(f_M - \epsilon_0 - \lambda_M) < \frac{1}{4}r_{\text{min}}$ and there are no modal-sets or $\epsilon_0$-modal sets with minimum density in $[\lambda_M, f_M - \epsilon_0)$.  
For each $A_i$, there exists exactly one CC of $\mathcal{X}^{\lambda_M}$, $A_i'$, such that
$A_i \subset A_i'$. Since $g(f_M - \epsilon_0 - \lambda_M) < \frac{1}{4}r_{\text{min}}$, it follows that $A_i' \subseteq B(A_i, \frac{1}{4}r_{\text{min}})$. Thus, the $A_i'$'s are pairwise separated by distance at least $\frac{1}{2}r_{\text{min}}$. Moreover, there are no other CCs in ${\mathcal{X}^{f_M - \epsilon_0}}$ because there are no modal-sets with density in $[\lambda_M, f_M - \epsilon_0)$. 

Then, let $A_M$ be the CC of $\mathcal{X}^{\lambda_M}$ containing $M$. Then $A_M$ contains no other $\epsilon_0$-modal sets and it is $\frac{1}{5}r_{\text{min}}$-separated by $\mathcal{X}^{\lambda_M} \backslash M$ by some set $S_M$ (i.e. take $S_M := \{x : d(x, A_M) = \frac{1}{5} r_{\text{min}}\}$). Since there is a finite number of modal-sets, it suffices to take $r_s$ to be the minimum of the corresponding $\frac{1}{5}r_{\text{min}}$ for each $\epsilon_0$-modal set. This resolves the first part of the proposition.

Let $h(r) := \inf_{x \in B(M, r)} f(x)$. Since $f$ is continuous, $h$ is continuous and decreasing with $h(0) = f_M - \epsilon_0 > \lambda_M$. Take $r_M > 0$ sufficiently small so that $h(r_M) > \lambda_M$. This resolves the second part of the proposition.

Take $M_0$ to be some modal-set with density $f_M$ in $M$. One must exist since $M$ has local-maxima at level $f_M$.
For each $r$, let $u_M(r) := \max\{f_M - \epsilon_0 - \inf_{x \in B(M, r)} f(x), f_M - \inf_{x \in B(M_0, r)} f(x) \}$. Then, we have $f_M - f(x) \le u_M(d(x, M_0))$ and $f_M - \epsilon_0 - f(x) \le u_M(d(x, M))$. Clearly $u_M$ is increasing on $[0, r_M]$ with $u_M(0) = 0$ and continuous since $f$ is continuous.  If $u_M$ is not strictly increasing then we can replace it with a strictly increasing continuous function while still having $u_M(r) \rightarrow 0$ as $r \rightarrow 0$ (i.e. by adding an appropriate strictly increasing continuous function). This resolves the third part of the proposition and the upper bound in the fourth part of the proposition. 

Now, define $g_M(t) := d({\mathcal{X}^{f_M - \epsilon_0 - t}} \cap {A_M}, M)$ 
for $t \in  [0, \frac{1}{2} (f_M - \epsilon_0 - \lambda_M)]$. 
Then, $g_M$ is continuous, $g_M(0) = 0$ and is strictly increasing. 
Define $l_M$ to be the inverse of $g_M$. Clearly $l_M$ is continuous, strictly increasing, and $l_M(r) \rightarrow 0$ as $r \rightarrow 0$. From the definition of $g_M$, it follows that for $x \in B(M, r_M)$, $f_M - \epsilon_0- f(x) \ge l_M(d(x, M))$
as desired.
\end{proof}

We need the following result giving guarantees on the empirical balls.
\begin{lemma}[\citep{CD10}] \label{ball_bounds} 
Pick $0 < \delta < 1$. Assume that $k \ge d \log n$. Then with probability at least $1 - \delta$, for every ball $B \subset \mathbb{R}^d$ we have
\begin{align*}
\mathcal{F}(B) \ge C_{\delta, n} \frac{\sqrt{d \log n}}{n} &\Rightarrow \mathcal{F}_n(B) > 0\\
\mathcal{F}(B) \ge \frac{k}{n} + C_{\delta, n} \frac{\sqrt{k}}{n} &\Rightarrow \mathcal{F}_n(B) \ge \frac{k}{n} \\
\mathcal{F}(B) \le \frac{k}{n} - C_{\delta, n}\frac{\sqrt{k}}{n} &\Rightarrow \mathcal{F}_n(B) < \frac{k}{n}.
\end{align*}
\end{lemma}

Lemma~\ref{fk_bounds} of \cite{dasgupta2014optimal} establish convergence rates for $f_k$. 

\begin{definition}\label{rhat} For $x \in \mathbb{R}^d$ and $\epsilon > 0$, define 
$\hat{r}(\epsilon, x):=\sup\left\{r : \sup_{x' \in B(x, r)} f(x') - f(x) \le \epsilon \right\}$  and
$\check{r}(\epsilon, x):=\sup\left\{r : \sup_{x' \in B(x, r)} f(x) - f(x') \le \epsilon \right\}$.
\end{definition}

\begin{lemma}[Bounds on $f_k$]\label{fk_bounds}  Suppose that $\frac{C_{\delta, n}}{\sqrt{k}} < \frac{1}{2}$. Then the follow two statements each hold with probability at least $1 - \delta$: 
\begin{align*}
f_k(x) < \left(1 + 2\frac{C_{\delta, n}}{\sqrt{k}} \right)(f(x) + \epsilon),
\end{align*}
for all $x\in \mathbb{R}^d$ and all $\epsilon > 0$ provided $k$ satisfies $v_d\cdot \hat{r}(\epsilon, x)^d \cdot (f(x) + \epsilon) \ge \frac{k}{n} - C_{\delta, n}\frac{\sqrt{k}}{n}$.
\begin{align*}
f_k(x) \ge \left(1 - \frac{C_{\delta, n}}{\sqrt{k}} \right)(f(x) - \epsilon),
\end{align*}
for all $x\in \mathbb{R}^d$ and all $\epsilon > 0$  provided $k$ satisfies $v_d\cdot \check{r}(\epsilon, x)^d \cdot (f(x) - \epsilon) \ge \frac{k}{n} + C_{\delta, n}\frac{\sqrt{k}}{n}$. 
\end{lemma}

\begin{lemma}[Extends Lemma~\ref{r_n_upper_bound}] \label{r_n_upper_bound_general}(Upper bound on $r_n$)  Let $M$ be an $\epsilon_0$-modal set with maximum density $f_M$ and suppose that $k$ is admissible. With probability at least $1 - \delta$,
\begin{align*}
r_n(M) \le  \left(\frac{2C_{\delta, n} \sqrt{d \log n}}{n\cdot v_d\cdot (f_M - \epsilon_0)}\right)^{1/d}.
\end{align*}
\end{lemma}

\begin{proof}[Proof of Lemma~\ref{r_n_upper_bound_general}]
Define $r_0 := \left(\frac{2C_{\delta, n} \sqrt{d \log n}}{nv_d\cdot (f_M - \epsilon_0)}\right)^{1/d}$ and $r := (4k/(nv_df_M))^{1/d}$. Since $k$ is admissible, we have that $u_M(r_0) \le u_M(r) \le (f_M  - \epsilon_0)/ 2$. We have
\begin{align*}
\mathcal{F}(B(x, r_0)) &\ge v_d{r_0}^d(f_M -\epsilon_0 - u_M(r_0)) \ge  v_d{r_0}^d (f_M - \epsilon_0)/2 
= \frac{C_{\delta, n} \sqrt{d\log n}}{n}.
\end{align*}
By Lemma~\ref{ball_bounds}, this implies that $\mathcal{F}_n (B(x, r_0)) > 0$ with probability at least $1 - \delta$ and therefore we have $r_n(x) \le r_0$.
\end{proof}

\section{Isolation Results} \label{appendix:isolation}

The following extends Lemma~\ref{isolation} to handle more general
$\epsilon_0$-modal sets and pruning parameter $\tilde\epsilon$. 

\begin{lemma}[Extends Lemma~\ref{isolation}] (Isolation) \label{isolation_general} 
%\textbf{Restatement of Lemma~\ref{isolation}} (Isolation)\textbf{.} 
Let $M$ be an $\epsilon_0$-modal set and $k$ be admissible for $M$.  Suppose $0 \le \tilde\epsilon < l_M(\min\{r_M, r_s\}/2)$ and let $\hat{x}_M := \argmax_{x \in \mathcal{X}_M \cap X_{[n]}} f_k(x)$. Then the following holds
with probability at least $1-5\delta$: when processing sample point $\hat{x}_M$ in Algorithm~\ref{alg:epsilonmodalset} we will add $\widehat{M}$ to $\widehat{\mathcal{M}}$ where  
$\widehat{M}$ does not contain points outside of $\mathcal{X}_M$.
\end{lemma}

\begin{proof}  Define  $\widehat{f}_M := f_k(\hat{x}_M)$, $\lambda = \widehat{f}_M$ and $\bar{r} := \min\{r_M, r_s\} / 2$.
It suffices to show that (\rm{i}) $\mathcal{X} \backslash \mathcal{X}_M$ and $B(M, \bar{r})$ are disconnected in $G(\lambda -  9\beta_k \lambda - \epsilon_0 - \tilde{\epsilon})$ and (\rm{ii}) $\hat{x}_M \in B(M, \bar{r})$. 

In order to show (\rm{i}), we first show that $G(\lambda - 9\beta_k \lambda - \epsilon_0 - \tilde{\epsilon})$ contains no points from $B(S_M, r_s/2)$ and no points from $\mathcal{X}_M \backslash B(M, \bar{r})$. Then, all that will be left is showing that there are no edges between $B(M, \bar{r})$ and $\mathcal{X} \backslash \mathcal{X}_M$.

We first prove bounds on $f_k$ that will help us show (\rm{i}) and (\rm{ii}). Let $\bar{F} := f_M - \epsilon_0 - l_M(\bar{r}/2)$. Then for all $x \in \mathcal{X}_M \backslash B(M, \bar{r})$, we have $\hat{r}(\bar{F} - f(x), x) \ge \bar{r}/2$. Thus the conditions for Lemma~\ref{fk_bounds} are satisfied by the admissibility of $k$ and hence $f_k(x) < \left(1 + 2\frac{C_{\delta, n}}{\sqrt{k}} \right) \bar{F}$. Now,
\begin{align*}
\sup_{x \in \mathcal{X}_M\backslash B(M, \bar{r})} f_k(x) &< (1 + 2\frac{C_{\delta, n}}{\sqrt{k}}) \bar{F} = (1 + 2\frac{C_{\delta, n}}{\sqrt{k}}) (f_M - \epsilon_0 - l_M(\bar{r}/2))\\
 &\le (1 + 2\frac{C_{\delta, n}}{\sqrt{k}})^3 \widehat{f}_M - (1 + 2\frac{C_{\delta, n}}{\sqrt{k}}) \cdot (\epsilon_0 + l_M(\bar{r}/2)) \le \lambda-  9\beta_k \lambda - \epsilon_0 - \tilde{\epsilon},
\end{align*}
where the second inequality holds by using Lemma~\ref{fk_bounds} as follows. Choose $x \in M_0$ and $\epsilon = \frac{C_{\delta, n}}{2\sqrt{k}} f_M$. Then $\check{r}(\epsilon, x) \ge u^{-1}(\epsilon)$. The conditions for Lemma~\ref{fk_bounds} hold by the admissibility of $k$ and thus $\widehat{f}_M \ge f_k(x) \ge (1 - C_{\delta, n}/\sqrt{k})^2 f_M$. Furthermore it follows from Lemma~\ref{fk_bounds} that $\widehat{f}_M <  (1 + 2C_{\delta, n}/\sqrt{k}) f_M$; 
combine this admissibility of $k$ to obtain the last inequality. Finally, from the above, we also have $\sup_{x \in \mathcal{X}_M\backslash B(M, \bar{r})} f_k(x) < \widehat{f}_M$, implying (\rm{ii}).

\noindent Next, if $x \in B(S_M, r_s/2)$, then $\hat{r}(\bar{F} - f(x) , x) \ge \bar{r}/2$ and the same holds for $B(S_M, r_s/2)$:
\begin{align*}
\sup_{x \in B(S_M, r_s/2)} f_k(x) < 
\lambda- 9\beta_k \lambda - \epsilon_0 -\tilde{\epsilon}.
\end{align*}
Thus, $G(\lambda- 9\beta_k \lambda - \epsilon_0 - \tilde{\epsilon})$ contains no point from $B(S_M, r_s/2)$ and no point from $\mathcal{X}_M \backslash B(M, \bar{r})$. 

\noindent All that remains is showing that there is no edge between $B(M, \bar{r})$ and $\mathcal{X} \backslash \mathcal{X}_M$. It suffices to show that any such edge will have length less than $r_s$ since $B(S_M, r_s/2)$ separates them by a width of $r_s$. We have for all $x \in B(M, \bar{r})$, 
\begin{align*}\mathcal{F}(B(x, \bar{r})) \ge v_d\bar{r}^d\inf_{x' \in B(x, 2\bar{r})} f(x') \ge \frac{k}{n} + C_{\delta, n}\frac{\sqrt{k}}{n}.
\end{align*}
Thus by Lemma~\ref{ball_bounds}, we have $r_k(x) \le \bar{r} < r_s$, establishing (\rm{i}).

\end{proof}

\section{Integrality Results} \label{appendix:integrality}

The goal is to show that the $\widehat{M} \in \widehat{\mathcal{M}}$ refered to above contains $B(M, r_n(M))$. We give a condition under which $B(M, r_n(M)) \cap X_{[n]}$ would be connected in $G(\lambda)$ for some $\lambda$. It is adapted from arguments in Theorem V.2 in \cite{CDKvL14}.

%\textbf{Lemma~\ref{connectedness}} (Connectedness)\textbf{.}
\begin{lemma}\label{connectedness} (Connectedness)
Let $M$ be an $\epsilon_0$-modal set and $k$ be admissible for $M$. Then with probability at least $1 - \delta$, $B(M, r_n(M)) \cap X_{[n]}$ is connected in $G(\lambda)$ if 
\begin{align*}
\lambda \le \left(1 - \frac{C_{\delta, n}}{\sqrt{k}}\right)^2 (f_M - \epsilon_0).
\end{align*}
\end{lemma}

\begin{proof}
For simplicity of notation, let $A := B(M, r_n(M))$. It suffices to prove the result for $\lambda = (1 - \frac{C_{\delta, n}}{\sqrt{k}})^2 (f_M - \epsilon_0)$.
Define $r_{\lambda} = (k/(nv_d\lambda))^{1/d}$ and $r_o = (k/(2nv_df_M))^{1/d}$. 
First, we show that each $x \in B(A, r_\lambda)$, there is a sample point in $B(x, r_o)$. 
We have for $x\in B(A, r_\lambda)$,
\begin{align*}
\mathcal{F}(B(x, r_o)) &\ge v_d r_o^d \inf_{x' \in B(x, r_o + r_\lambda )} f(x') \ge v_d r_o^d(f_M - \epsilon_0 - u_M(r_o + r_\lambda + r_n(M)))  \\
&\ge v_d r_o^d (f_M - \epsilon_0) \left(1 - \frac{C_{\delta, n}}{\sqrt{k}}\right) \ge C_{\delta, n} \frac{\sqrt{d \log n}}{n}.
\end{align*}
Thus by Lemma~\ref{ball_bounds} we have that with probability at least $1 - \delta$, $B(x, r_o)$ contains a sample uniformly over $x \in B(A, r_\lambda)$. 

 \noindent Now, let $x$ and $x'$ be two points in $A\cap X_{[n]}$. We now show that there exists $x = x_0,x_1,...,x_p = x'$ such that $||x_i - x_{i+1}|| < r_o$ and $x_i \in B(A, r_o)$. Note that since $A$ is connected and the density in $B(A, r_o + r_\lambda)$ is lower bounded by a positive quantity, then for arbitrary $\gamma \in (0, 1)$, we can choose $x = z_0, z_1,...,z_p = x'$ where $||z_{i+1} - z_i|| \le \gamma r_o$. Next, choose $\gamma$ sufficiently small such that 
\begin{align*}
v_d\left( \frac{(1-\gamma)r_o}{2}\right)^d\gamma \ge \frac{C_{\delta, n}\sqrt{d\log n}}{n},
\end{align*}
then there exists a sample point $x_i$ in $B(z_i, (1-\gamma)r_o/2)$. Moreover we obtain that 
\begin{align*}
||x_{i+1} - x_i|| &\le ||x_{i+1} - z_{i+1}|| + ||z_{i+1} - z_i|| + ||z_i -x_i|| \le r_o.
\end{align*} 

\noindent All that remains is to show $(x_i, x_{i+1}) \in G(\lambda)$. We see that $x_i \in B(A, r_o)$. However, for each $x \in B(A, r_o)$, we have 
\begin{align*}
\mathcal{F}(B(x, r_\lambda)) &\ge v_dr_\lambda^d \inf_{x' \in B(x, r_o + r_\lambda)} f(x') 
\ge v_d r_\lambda^d (f_M - \epsilon_0) \left(1 - \frac{C_{\delta, n}}{\sqrt{k}}\right)
\ge \frac{k}{n} + \frac{C_{\delta, n} \sqrt{k}}{n}.
\end{align*}
 Thus $r_k(x_i) \le r_\lambda$ for all $i$. Therefore, $x_i \in G(\lambda)$ for all $x_i$. Finally, $||x_{i+1} - x_i|| \le r_o \le \min \{ r_k(x_i), r_k(x_{i+1})\}$ and thus $(x_i, x_{i+1}) \in G(\lambda)$. Therefore, $A \cap X_{[n]}$ is connected in $G(\lambda)$, as desired. 
 \end{proof}

The following extends Lemma~\ref{integrality} handle more general $\epsilon_0$-modal sets. 
%\textbf{Restatement of Lemma~\ref{integrality}} (Integrality)\textbf{.}
\begin{lemma}[Extends Lemma~\ref{integrality}] (Integrality) \label{integrality_general} 
Let $M$ be an $\epsilon_0$-modal set with density $f_M$, and suppose $k$ is admissible for $M$. Let $\hat{x}_M := \argmax_{x \in \mathcal{X}_M \cap X_{[n]}} f_k(x)$. Then the following holds with probability at least $1 - 3\delta$. When processing sample point $\hat{x}_M$ in Algorithm~\ref{alg:modalset}, if we add $\widehat{M}$ to $\widehat{\mathcal{M}}$, then $B(M, r_n(M)) \cap X_{[n]} \subseteq \widehat{M}$.
\end{lemma}

\begin{proof}  Define $\widehat{f}_M := f_k(\hat{x}_M)$ and $\lambda := \widehat{f}_M$. It suffices to show that $B(M, r_n(M)) \cap X_{[n]}$ is connected in $G(\lambda - 9\beta_k \lambda - \tilde{\epsilon})$. By Lemma~\ref{connectedness}, $B(M, r_n(M)) \cap X_{[n]}$ is connected in $G(\lambda_0)$ when $\lambda_0 \le (1 - \frac{C_{\delta, n}}{\sqrt{k}})^2 (f_M - \epsilon_0) $. Indeed, we have that
\begin{align*}
\left(1 - \frac{C_{\delta, n}}{\sqrt{k}}\right)^2 (f_M  - \epsilon_0)
&\ge \widehat{f}_M \left(1 - \frac{C_{\delta, n}}{\sqrt{k}}\right)^2 / \left(1 + 2\frac{C_{\delta, n}}{\sqrt{k}}\right) -  \left(1 - \frac{C_{\delta, n}}{\sqrt{k}}\right)^2 \epsilon_0 \\
&\ge \lambda -  \beta_k \lambda - \epsilon_0 \ge \lambda -  9\beta_k \lambda - \epsilon_0 -  \tilde{\epsilon},
\end{align*}
where the first inequality follows from Lemma~\ref{fk_bounds},
as desired.
\end{proof}

\section{Theorem~\ref{theo:main}} \label{appendix:theomain}

Combining the  isolation and integrality, we obtain the following extention of Corollary~\ref{identification}.

%\textbf{Restatement of Corollary~\ref{identification}} (Identification)\textbf{.}

\begin{corollary}[Extends Corollary~\ref{identification}] (Identification) \label{identification_general}
Suppose we have the assumptions of Lemmas~\ref{isolation_general} and~\ref{integrality_general} for $\epsilon_0$-modal set $M$. Define $\widehat{f}_M := \max_{x \in \mathcal{X}_M \cap X_{[n]}} f_k(x)$ and $\lambda := \widehat{f}_M$. With probability at least $1 - 5\delta$, there exists $\widehat{M} \in \widehat{\mathcal{M}}$ such that $B(M, r_n(M)) \cap X_{[n]} \subseteq  \widehat{M} \subseteq \{ x \in \mathcal{X}_M \cap X_{[n]} : f_k(x) \ge \lambda- \beta_k \lambda  - \epsilon_0 \}$
\end{corollary}

\begin{proof}
By Lemma~\ref{isolation_general}, there exists $\widehat{M} \in  \widehat{\mathcal{M}}$ which contains only points in $\mathcal{X}_M$ with maximum $f_k$ value of $\widehat{f}_M$. Thus, we have $\widehat{M} \subseteq \{ x \in \mathcal{X}_M \cap X_{[n]} : f_k(x) \ge \widehat{f}_M - \beta_k \widehat{f}_M - \epsilon_0 \}$. By Lemma~\ref{integrality_general}, $B(M, r_n(M)) \cap X_{[n]} \subseteq  \widehat{M}$.
\end{proof}

The following extends Theorem~\ref{theo:main} to handle more general $\epsilon_0$-modal sets and pruning parameter $\tilde{\epsilon}$. 

\begin{theorem}[Extends Theorem~\ref{theo:main}] \label{theo:main_general}
Let $\delta > 0$ and $M$ be an $\epsilon_0$-modal set. Suppose $k$ is admissible for $M$ and $0 \le \tilde\epsilon < l_M(\min\{r_M, r_s\}/2)$. Then with probability at least $1 - 6\delta$, there exists $\widehat{M} \in \widehat{\mathcal{M}}$ such that 
 \begin{align*}
 d(M, \widehat{M}) \le l_M^{-1}\left(\frac{8C_{\delta,n }}{\sqrt{k}}f_M\right), \end{align*}
 which goes to $0$ as $C_{\delta, n}/\sqrt{k} \rightarrow 0$.
\end{theorem}

\begin{proof} Define $\tilde{r} = l_M^{-1}\left(\frac{8C_{\delta,n }}{\sqrt{k}}f_M\right)$. There are two directions to show: $\max_{x \in \widehat{M}} d(x, M) \le \tilde{r}$ and $\sup_{x \in M} d(x, \widehat{M}) \le \tilde{r}$ with probability at least $1 - \delta$.

We first show $\max_{x \in \widehat{M}} d(x, M) \le \tilde{r}$.
By Corollary~\ref{identification_general} we have $\widehat{M} \in \widehat{\mathcal{M}}$ such that $\widehat{M} \subseteq \{ x \in \mathcal{X}_M : f_k(x) \ge \widehat{f}_M - \beta_k \widehat{f}_M - \epsilon_0 \}$ where $\widehat{f}_M := \max_{x \in \mathcal{X}_M \cap X_{[n]}} f_k(x)$. 
Hence, it suffices to show 
\begin{align}\label{consistency_to_show}
\inf_{x \in B(M_0, r_n(M))} f_k(x) \ge \sup_{\mathcal{X}_M \backslash B(M, \tilde{r})} f_k(x) + \beta_k \widehat{f}_M + \epsilon_0.
\end{align}
Define $r := (4/f_Mv_d)^{1/d}(k/n)^{1/d}$. For any $x \in B(M_0, r + r_n(M))$, $f(x) \ge f_M - u_M(r + r_n(M)) := \check{F}$. Thus, for any $x \in B(M_0, r_n(M))$ we can let $\epsilon = f(x) - \check{F}$ and thus $\check{r}(\epsilon, x) \ge r$ and hence the conditions for Lemma~\ref{fk_bounds} are satisfied. Therefore, with probability at least $1-\delta$,
\begin{align}\label{consistency_bound1}
 \inf_{x \in B(M_0, r_n(M))} f_k(x) \ge \left(1 - \frac{C_{\delta, n}}{\sqrt{k}}\right)(f_M - u_M(r +r_n(M))).
\end{align}
For any $x\in \mathcal{X}_M \backslash B(M, \tilde{r}/2)$, $f(x) \le f_M - \epsilon_0 - l_M(\tilde{r}/2) := \hat{F}$. Now, for any $x \in \mathcal{X} \backslash B(M, \tilde{r})$, let $\epsilon := \hat{F} - f(x)$. We have $\hat{r}(\epsilon, x) \ge \tilde{r}/2 = l^{-1}_M (8C_{\delta, n}/\sqrt{k})/2 \ge l^{-1}_M (u_M(2r)) / 2 \ge r$ (since $l_M$ is increasing and $l_M \le u_M$) and thus the conditions for Lemma~\ref{fk_bounds} hold. Hence, with probability at least $1 - \delta$,
\begin{align}\label{consistency_bound2}
\sup_{x \in \mathcal{X}_M \backslash B(M, \tilde{r})} f_k(x) \le \left(1 + 2\frac{C_{\delta, n}}{\sqrt{k}}\right)(f_M - \epsilon_0 -  l_M(\tilde{r})).
\end{align}
Thus, by (\ref{consistency_bound1}) and (\ref{consistency_bound2}) applied to (\ref{consistency_to_show}) it suffices to show that
\begin{align}\label{consistency_to_show_2}
&\left(1 - \frac{C_{\delta, n}}{\sqrt{k}}\right)(f_M - u_M(r + r_n(M))) 
\ge \left(1 + 2\frac{C_{\delta, n}}{\sqrt{k}}\right)(f_M - \epsilon_0 - l_M(\tilde{r})) + \beta_k \widehat{f}_M + \epsilon_0,
\end{align}
which holds when 
\begin{align}\label{consistency_to_show_3}
l_M(\tilde{r}) \ge u_M(r + r_n(M)) + \frac{3C_{\delta, n}}{\sqrt{k}}f_M + \beta_k \widehat{f}_M.
\end{align}
The admissibility of $k$ ensures that $r_n(M) \le r \le r_M/2$ so that the regions of $\mathcal{X}$ we are dealing with in this proof are confined within $B(M_0, r_M)$ and $B(M, r_M) \backslash M$.

By the admissibility of $k$, $u_M(2r) \le \frac{C_{\delta, n}}{2\sqrt{k}}f_M$. This gives
\begin{align*}
l_M(\tilde{r}) &= \frac{8C_{\delta,n }}{\sqrt{k}}f_M \ge u_M(2r) + \frac{15C_{\delta,n }}{2\sqrt{k}}f_M 
%&\ge u_M(r + r_n) + 3f_M + 4\left(1 + 2\frac{C_{\delta, n}}{\sqrt{k}}\right)f_M \\
\ge u_M(r + r_n(M)) + \frac{3C_{\delta, n}}{\sqrt{k}} f_M + \beta_k \widehat{f}_M,
\end{align*}
where the second inequality holds since $C_{\delta, n}/\sqrt{k} < 1/16$, $u$ is increasing, $r \ge r_n(M)$, and $\widehat{f}_M \le \left(1 + 2\frac{C_{\delta, n}}{\sqrt{k}}\right)f_M$ by Lemma~\ref{fk_bounds}. Thus, showing (\ref{consistency_to_show_3}), as desired.

This shows one direction of the Hausdorff bound. We now show the other direction, that $\sup_{x \in M} d(x, M) \le \tilde{r}$.

It suffices to show for each point $x \in M$ that the distance to the closest sample point $r_n(x) \le \tilde{r}$ since $\widehat{M}$ contains these sample points by Corollary~\ref{identification_general}. However, by Lemma~\ref{r_n_upper_bound_general} and the admissibility of $k$, $r_n(x) \le \tilde{r}$  as desired.
\end{proof}

\section{Theorem~\ref{pruning}} \label{appendix:pruning}

We need the following Lemma~\ref{connectedness_pruning} which gives guarantees us that 
given points in separate CCs of the pruned graph, these points will also be in separate CCs of $f$ at a nearby level. \cite{CDKvL14} gives a result for a different graph and the proof can be  adapted to give the same result for our graph (but slightly different assumptions on $k$).

\begin{lemma}[Separation of level sets under pruning, \cite{CDKvL14}]  \label{connectedness_pruning} 
Fix $\epsilon > 0$ and let $r(\epsilon) := \inf_{x \in \mathbb{R}^d} \min \{\hat{r}(\epsilon, x), \check{r}(\epsilon, x) \}$. Define $\Lambda := \max_{x \in \mathbb{R}^d} f(x)$ and assume $\tilde\epsilon_0 \ge 2 \epsilon + \beta_k(\lambda_f + \epsilon)$ and let $\tilde{G}(\lambda)$ be the graph with vertices in $G(\lambda)$ and edges between pairs of vertices if they are connected in $G(\lambda - \tilde\epsilon_0)$. Then the following holds with probability at least $1 - \delta$.\\

Let $\tilde{A}_1$ and $\tilde{A}_2$ denote two disconnected sets of points $\tilde{G}(\lambda)$. Define $\lambda_f := \inf_{x \in \tilde{A}_1 \cup \tilde{A}_2}f(x)$.  Then $\tilde{A}_1$ and $\tilde{A}_2$ are disconnected in the level set $\{ x \in \mathcal{X} : f(x) \ge \lambda_f\}$ if $k$ satisfies
\begin{align*}
v_d (r(\epsilon) / 2) ^d (\lambda_f - \epsilon) \ge \frac{k}{n} + C_{\delta, n}\frac{\sqrt{k}}{n}
\end{align*}
and 
\begin{align*}
k \ge \max \{ 8 \Lambda C_{\delta, n}^2/(\lambda_f - \epsilon), 2^{d + 7} C_{\delta, n}^2 \}.
\end{align*}
\end{lemma}

\begin{proof}
We prove the contrapositive. Let $A$ be a CC of $\{x \in \mathcal{X} : f(x) \ge \lambda_f \}$ with $\lambda_f = \min_{x \in A \cap X_{[n]}} f(x)$. Then it suffices to show $A \cap X_{[n]}$ is connected in $G(\lambda')$ for $\lambda' := \min_{x \in A \cap X_{[n]}} f_k(x) - \tilde\epsilon_0$. 

We first show $A \cap X_{[n]}$ is connected in $G(\lambda)$ for $\lambda = (\lambda_f - \epsilon) / (1 + C_{\delta, n}/\sqrt{k})$ and all that will remain is showing $\lambda' \le \lambda$.

Define $r_o := (k/(2nv_df_M))^{1/d}$ and $r_\lambda := (k/(nv_d\lambda))^{1/d}$. Then from the first assumption on $k$, it follows that $r_\lambda \le r(\epsilon)/2$. Now for each $x \in B(A, r_\lambda)$, we have 
\begin{align*}
\mathcal{F}(B(x, r_o)) \ge v_d r_o^d \inf_{x' \in B(x, r_o + r_\lambda )} f(x') 
\ge v_d r_o^d(\lambda_f - \epsilon) \ge C_{\delta, n} \frac{\sqrt{d \log n}}{n}.
\end{align*}
Thus, by Lemma~\ref{ball_bounds} we have with probability at least $1-\delta$ that $B(x, r_0)$ contains a sample point.

Now, in the same way shown as in Lemma~\ref{connectedness}, we have the following. If $x$ and $x'$ be two points in $A\cap X_{[n]}$ then there exists $x = x_0,x_1,...,x_p = x'$ such that $||x_i - x_{i+1}|| < r_o$ and $x_i \in B(A, r_o)$. 

Next is showing $(x_i, x_{i+1}) \in G(\lambda)$. We see that $x_i \in B(A, r_o)$. However, for each $x \in B(A, r_o)$, we have 
\begin{align*}
\mathcal{F}(B(x, r_\lambda)) &\ge v_dr_\lambda^d \inf_{x' \in B(x, r_o + r_\lambda)} f(x') 
\ge v_d r_\lambda^d (\lambda_f - \epsilon) \ge \frac{k}{n} + \frac{C_{\delta, n} \sqrt{k}}{n}.
\end{align*}
 Thus $r_k(x_i) \le r_\lambda$ for all $i$. Therefore, $x_i \in G(\lambda)$ for all $x_i$. Finally, $||x_{i+1} - x_i|| \le r_o \le \min \{ r_k(x_i), r_k(x_{i+1})\}$ and thus $(x_i, x_{i+1}) \in G(\lambda)$. Therefore, $A \cap X_{[n]}$ is connected in $G(\lambda)$.

All that remains is showing $\lambda' \le \lambda$. We have
\begin{align*}
\lambda' = \min_{x \in A \cap X_{[n]}} f_k(x) - \tilde\epsilon_0 
\le \left(1 + 2\frac{C_{\delta, n}}{\sqrt{k}} \right) (\lambda_f + \epsilon) - \tilde\epsilon_0
\le \lambda,
\end{align*}
where the first inequality holds by Lemma~\ref{fk_bounds}, and the second inequality holds from the assumption on $\tilde\epsilon_0$, as desired. 

\end{proof}

We state the pruning result for more general choices of $\tilde{\epsilon}$. Its proof is standard and given here for completion. (See e.g. \cite{dasgupta2014optimal}). 

%\textbf{Restatement of Theorem~\ref{pruning}.} 
\begin{theorem}[Extends Theorem~\ref{pruning}] \label{pruning_general}
Let $0< \delta < 1$ and $\tilde\epsilon \ge 0$. There exists $\lambda_0 = \lambda_0(n, k)$ such that the following holds with probability at least $1 - \delta$. All $\epsilon_0$-modal set estimates in $\widehat{\mathcal{M}}$ chosen at level $\lambda \ge \lambda_0$ can be injectively mapped to $\epsilon_0$-modal sets $\braces{M:  \lambda_M \geq \min_{\{x\in \mathcal{X}_{[n]} : f_k(x) \ge \lambda - \beta_k \lambda \}} f(x)}$, provided $k$ is admissible for all such $M$.  

In particular, if $f$ is H\"older-continuous, (i.e. $||f(x) - f(x')|| \le c||x - x'||^\alpha $ for some $0 < \alpha\le 1$, $c > 0$) and $\tilde\epsilon = 0$, 
then $\lambda_0 \to 0$ as $n\to \infty$, provided 
$C_1 \log n \le k \le C_2 n^{2\alpha / (2\alpha + d)}$, for some $C_1, C_2$ independent of $n$. 
%Let $0< \delta < 1$ and take $\lambda > 0$. Define 
%\begin{align*}
%\lambda_0 := \max \left\{ 4\tilde\epsilon(\lambda), 16 C_{\delta, n}^2 \sup_{x \in \mathcal{X}} f(x) / k + 2\tilde\epsilon(\lambda), \frac{5k}{n\cdot v_d\cdot r(\tilde\epsilon(\lambda)/3)^d} \right\}.
%\end{align*}
%Then following holds with probability at least $1 - \delta$.  Suppose that $\tilde\epsilon(\lambda) \ge 8\beta_k\lambda$.  If $\lambda \ge \lambda_0$, then all $\epsilon_0$-modal set estimates in $\widehat{\mathcal{M}}$ chosen at level $\lambda$ can be injectively mapped to $\epsilon_0$-modal sets $\braces{M:  \lambda_M \geq \min_{\{x\in \mathcal{X}_{[n]} : f_k(x) \ge \lambda - \beta_k \lambda \}} f(x)}$, provided $k$ is admissible for all such $M$. 
\end{theorem}

\begin{proof}
Define $r(\epsilon) := \inf_{x \in \mathbb{R}^d} \min \{\hat{r}(\epsilon, x), \check{r}(\epsilon, x) \}$. Since $f$ is uniformly continuous, it follows that $r(0) = 0$, $r$ is increasing, and $r(\epsilon) > 0$ for $\epsilon > 0$. 

Thus, there exists $\tilde{\lambda}_{n,k,\tilde\epsilon} > 0$ such that 
\begin{align*}
\tilde{\lambda}_{n, k, \tilde\epsilon} = \frac{k}{n\cdot v_d \cdot r((8\beta_k \tilde{\lambda}_{n,k,\tilde\epsilon} + \tilde{\epsilon})/3)}.
\end{align*}
Define
\begin{align*}
\lambda_0 := \max \{\tilde\lambda_{n,k,\tilde\epsilon}, 32\beta_k \sup_{x \in \mathcal{X}} f(x) + 4 \tilde\epsilon \}.
\end{align*}

Let us identify each estimated $\epsilon_0$-modal set $\widehat{M}$ with the point $\hat{x}_M := \max_{x \in \widehat{M}} f_k(x)$. Let us call these points modal-points. Then it suffices to show that there is an injection from modal-points to the $\epsilon_0$-modal sets. 

Define $G'(\lambda)$ to be the graph with vertices in $G(\lambda - \beta_k \lambda)$ and edges between vertices if they are in the same CC of $G(\lambda - 9\beta_k \lambda - \tilde\epsilon)$ and $X_{[n]}^\lambda := \{ x : f_k(x) \ge \lambda \}$.  
Let $\tilde{A}_{i, \lambda} := \tilde{A}_i \cap X_{[n]}^\lambda$ for $i = 1,...,m$ to be the vertices of the CCs of $G'(\lambda)$ which do not contain any modal-points chosen thus far as part of estimated modal-sets. 

Fix level $\lambda > 0$ such that $\lambda_f := \inf_{x \in X_{[n]}^\lambda} f(x) \ge \lambda_0/2$. Then the conditions are satisified for Lemma~\ref{connectedness_pruning} with $\epsilon = (8\beta_k\lambda + \tilde\epsilon)/3$. Suppose that $\tilde{A}_{1,\lambda},...,\tilde{A}_{m,\lambda}$ are in ascending order according to $\lambda_{i, f} := \min_{x \in \tilde{A}_{i, \lambda}} f(x)$. Starting with $i = 1$, by Lemma~\ref{connectedness_pruning}, $\mathcal{X}^{\lambda_1, f}$ can be partitioned into disconneced subsets $A_1$ and $\mathcal{X}^{\lambda_1, f} \backslash A_1$ containing respectively $\tilde{A}_{1, \lambda}$ and $\cup_{i=2}^m \tilde{A}_{i, \lambda}$. Assign the modal-point $\argmax_{x \in \tilde{A}_{1, \lambda}} f_k(x)$ to any $\epsilon_0$-modal set in $A_1$. Repeat the same argument successively for any $\tilde{A}_{i, \lambda}$ and $\cup_{j=i+1}^m \tilde{A}_{j, \lambda}$ until all modal-points are assigned to distinct $\epsilon_0$-modal sets in disjoint sets $A_i$. 

Now by Lemma~\ref{connectedness_pruning}, $\mathcal{X}^{\lambda_f}$ can be partitioned into disconnected subsets $A$ and $\mathcal{X}^{\lambda_f} \backslash A$ containing respectively $\tilde{A}_\lambda := \cup_{i=1}^m \tilde{A}_{i, \lambda}$ and $\mathcal{X}_{[n]}^{\lambda_f} \backslash \tilde{A}_{\lambda}$. Thus, the modal-points in $\tilde{A}_\lambda$ were assigned to $\epsilon_0$-modal sets in $A$. 

\noindent Now we repeat the argument for all $\lambda' > \lambda$ to show that the modal-points  in $X_{[n]}^\lambda \backslash \tilde{A}_\lambda$ can be assigned to distinct $\epsilon_0$-modal sets in $\mathcal{X}^\lambda \backslash A^\lambda$. (We have $\lambda'_f := \min_{x \in X_{[n]}^{\lambda' - \beta_k\lambda'} }f(x) \ge \lambda_f$).

\noindent Finally, it remains to show that $\lambda \ge \lambda_0$ implies $\lambda_f \ge \lambda_0 /2$. We have $\lambda_0 / 4 \ge 8\beta_k\lambda + \tilde\epsilon$, thus $r(\lambda_0 / 4) \ge r(8\beta_k\lambda + \tilde\epsilon)$. It follows that
\begin{align*}
v_d(r(\lambda_0/4))^d\cdot (\lambda_0/4) \ge \frac{k}{n} + C_{\delta, n}\frac{\sqrt{k}}{n}.
\end{align*} 
Hence, for all $x$ such that $f(x) \le \lambda_0 / 2$, we have
\begin{align*}
f_k(x) \le (1 + 2\frac{C_{\delta, n}}{\sqrt{k}}) (f(x) + \lambda_0 / 4) \le \lambda_0.
\end{align*}

To see the second part, suppose we have $C_1, C_2 > 0$ such that $C_1 \log n \le k \le C_2 n^{2\alpha / (2\alpha + d)}$. This combined with the fact that $r(\epsilon) \ge (\epsilon/C)^{1/\alpha}$ implies $\lambda_0 \rightarrow 0$, as desired.

\end{proof}

\section{Point-Cloud Density} \label{appendix:pointcloud}
Here we formalize the fact that modal-sets can serve as good models for high-density structures in data, for instance 
a low-dimensional structure $M$ $+$ noise. 

\begin{lemma} (Point Cloud with Gaussian Noise)
Let $M \subseteq \mathbb{R}^d$ be compact (with possibly multiple connected-components of differing dimension). Then there exists a density $f$ over $\mathbb{R}^d$ such that the density is uniform in $M$ and has Gaussian decays around $M$ i.e.
\begin{align*}
f(x) = \frac{1}{Z} \exp(-d(x, M)^2/(2\sigma^2)),
\end{align*}  
where $\sigma > 0$ and $Z > 0$ depends on $M, \sigma$. Thus, the modal-sets of $f$ are the connected-components of $M$. 
\end{lemma}

\begin{proof}
Since $M$ is compact in $\mathbb{R}^d$, it is bounded. Thus there exists $R > 0$ such that $M \subseteq B(0, R)$. 
It suffices to show that for any $\sigma > 0$,
\begin{align*}
\int_{\mathbb{R}^d} \exp(-d(x, M)^2/(2\sigma^2)) dx < \infty.
\end{align*}
By a scaling of $x$ by $\sigma$, it suffices to show that 
\begin{align*}
\int_{\mathbb{R}^d} g(x) dx < \infty,
\end{align*}
where $g(x) := \exp(-\frac{1}{2} d(x, M)^2)$. Consider level sets $\mathcal{X}^\lambda := \{ x \in \mathbb{R}^d : g(x) \ge \lambda \}$. 
Note that $\mathcal{X}^\lambda \subseteq B(M, \sqrt{2 \log(1/\lambda)})$ based on the decay in $g$ around $M$. Clearly the image of $g$ is $(0, 1]$ so consider partitioning this range into intervals $[1, 1/2], [1/2, 1/3], ...$. Then it follows that 
\begin{align*}
\int_{\mathbb{R}^d} g(x) dx 
&\le \sum_{n=2}^\infty \text{Vol}(\mathcal{X}^{1/n}) \left(\frac{1}{n-1} - \frac{1}{n}\right)
\le \sum_{n=2}^\infty \frac{\text{Vol}(B(M, \sqrt{2 \log(n))}) }{(n-1)n}\\
&\le \sum_{n=2}^\infty \frac{\text{Vol}(B(0, R + \sqrt{2 \log(n))}) }{(n-1)n}
= \sum_{n=2}^\infty \frac{v_d(R + \sqrt{2 \log(n))})^d }{(n-1)n} \\
&\le v_d\cdot 2^{d-1} \sum_{n=2}^\infty \frac{R^d + (2 \log(n))^{d/2}}{(n-1)n} < \infty,
\end{align*}
where the last inequality holds by AM-GM. As desired.
\end{proof}

%\section{Additional Experiments}

%{\bf Toy datasets:} We compare M-cores to state-of-the-art procedures, under \emph{default settings},   
%on standard toy cluster shapes. This is motivated by common practice where clustering procedures are 
%used as \emph{black-box} routines, i.e. under default settings (implementing various rule of thumbs) provided by software packages; 
%this is because it is difficult to properly tune clustering procedures, given the absence of labels. 
%The data and procedures are from the popular \textit{sci-kit-learn} software package \cite{sklearndemo}. M-cores uses fixed settings as described earlier.  For the other procedures, we employed the automatic tuning used by the \textit{sci-kit-learn} package; K-Means and Spectral Clustering are additionally given the \emph{correct} number of clusters. 
%Results are shown in Figure~\ref{fig:toyclustering} along with running times. 
%M-cores compares well across the board, with relatively good time efficiency (all the procedures are in Python and C++).% similar or better time efficiency than Mean-Shift. Spectral Clustering is most expensive due to spectral decompositions. 

%\begin{figure}[h]
%\includegraphics[width=1\textwidth]{Figures/toy_clustering_2.png}
%\caption{Toy cluster-shapes handled by various procedures. The means or modes (used as cluster centers) by K-Means or Mean Shift 
%are shown as larger dotted points.}
%\label{fig:toyclustering}
%\end{figure} 

\section{Implementation} \label{appendix:implementation}

In this section, we explain how to implement Algorithm~\ref{alg:epsilonmodalset} (which supersedes Algorithm~\ref{alg:modalset}) efficiently. Here we assume that for our sample $\Xspl$, we have the $k$-nearest neighbors for each sample point. In our implementation, we simply use kd-tree, although one could replace it with any method that can produce the $k$-nearest neighbors for all the sample points. In particular, one could use approximate $k$-NN methods if scale is an issue. 

This section now concerns with what remains: constructing a data structure that maintains the CCs of the mutual $k$-NN graph as we traverse down the levels. At level $\lambda$ in Algorithm~\ref{alg:epsilonmodalset}, we must keep track of the mutual $k$-NN graph for points $x$ such that $f_k(x) \ge \lambda - 9\beta_k \lambda - \epsilon_0 - \tilde\epsilon$. Thus as $\lambda$ decreases, we add more vertices (and corresponding edges to the mutual $k$-nearest neighbors). Algorithm~\ref{alg:epsilonmodalsetinterface} shows what functions this data structure must support. Namely, adding nodes and edges, getting CCs of nodes, and checking if a CC intersects with the current estimates of the $\epsilon_0$-modal sets.

We implement this data structure as a disjoint-set forest data structure.  
The CCs can be represented as disjoint-sets of forests. Adding a node corresponds to making a set while adding an edge corresponds to a union operation. We can identify the verticies with the roots of the corresponding set's trees and thus getConnectedComponent and componentSeen can be implemented in a straightforward way.

In sum, the bulk of the time complexity is in preprocessing the data. This consists of obtaining the initial $k$-NN graph, i.e. distances to nearest neighbors; this one time operation is of worst-case order $O(n^2)$, similar to usual clustering procedures (e.g. Mean-Shift, K-Means, Spectral Clustering), but average case $O(nk \log n)$. After this preprocessing step, the estimation procedure itself requires just $O(nk)$ operations, each with amortized $O(\alpha(n))$ where $\alpha$ is the inverse Ackermann function. Thus, the implementation provided in Algorithm~\ref{alg:epsilonmodalsetimpl} is near-linear in $k$ and $n$.

\begin{algorithm}[tb]
   \caption{Interface for Mutual $k$-NN graph construction}
      \label{alg:epsilonmodalsetinterface}
\begin{algorithmic}
\STATE InitializeGraph()  \hfill // Creates an empty graph
\STATE addNode(G, node)   \hfill // Adds a node
\STATE addEdge(G, node1, node2) \hfill // Adds an edge
\STATE getConnectedComponent(G, node)   \hfill// Get the vertices in node's CC
\STATE componentSeen(G, node)       \hfill   // checks whether node's CC intersects with the estimates. If not, then marks the component as seen.
\end{algorithmic}
\end{algorithm}

\begin{algorithm}[tb]
   \caption{Implementation of M-cores (Algorithm~\ref{alg:epsilonmodalset})}
   \label{alg:epsilonmodalsetimpl}
\begin{algorithmic}

\STATE Let $\text{kNNSet}(x)$ be the $k$-nearest neighbors of $x \in \Xspl$. 
\STATE $\widehat{\mathcal{M}} \leftarrow \{ \}$
\STATE $G \leftarrow \text{InitializeGraph()}$ 
\STATE Sort points in descending order of $f_k$ values
\STATE Let $p \leftarrow 1$
\FOR{$i = 1,...,n$}
\STATE $\lambda \leftarrow f_k(X_i)$
\WHILE {$p < n$ and $f_k(X_p) \ge \lambda - 9\beta_k \lambda -\epsilon_0 - \tilde\epsilon$}
\STATE addNode($G, X_p$)
\FOR{$x \in \text{kNNSet}(X_p) \cap G$}
\STATE addEdge($G, x, X_p$)
\ENDFOR 
\STATE $p \leftarrow p + 1$
\ENDWHILE
\IF{not componentSeen($G, X_i$)}
\STATE $\text{toAdd} \leftarrow \text{getConnectedComponent}(G, X_i)$
\STATE Delete all $x$ from $\text{toAdd}$ where $f_k(x) < \lambda - \beta_k \lambda$
\STATE $\widehat{\mathcal{M}} \leftarrow \widehat{\mathcal{M}} + \{\text{toAdd}\}$
\ENDIF
\ENDFOR
\RETURN{$\widehat{\mathcal{M}}$}
\end{algorithmic}
\end{algorithm}

\end{document}